\documentclass{article}

\usepackage[final]{style/neurips_2023}

\PassOptionsToPackage{numbers, compress}{natbib}
\usepackage[utf8]{inputenc} %
\usepackage[T1]{fontenc}    %
\usepackage[hidelinks,colorlinks=true,citecolor=blue]{hyperref}       %
\usepackage{url}            %
\usepackage{booktabs}       %
\usepackage{amsfonts}       %
\usepackage{nicefrac}       %
\usepackage{microtype}      %
\usepackage{dsfont}
\usepackage{etoc}
\usepackage{minitoc}

\usepackage{pbox}
\usepackage{bigints}
\usepackage{url}
\usepackage{comment} 
\usepackage{booktabs}  

\usepackage{enumitem}

\usepackage{wrapfig}
\usepackage{tikz}
\usetikzlibrary{bayesnet}

\usepackage{algorithmic}
\usepackage[linesnumbered, ruled]{algorithm2e}
\usepackage{color}
\usepackage[tableposition=top]{caption}
\usepackage[titletoc,toc,title]{appendix}
\usepackage{subcaption}
\usepackage[utf8]{inputenc} %
\usepackage[T1]{fontenc}    %
\usepackage{amsfonts}       
\usepackage{amsthm}       
\usepackage{nicefrac}       
\usepackage{microtype}      
\usepackage{graphicx}
\usepackage[raggedright]{sidecap}
\usepackage{natbib}
\usepackage{xspace}
\usepackage{todonotes}
\usepackage{amsmath}
\usepackage{amsopn}
\usepackage{multirow}
\newcommand{\ci}{\perp\!\!\!\perp}

\usepackage{tikz}
\usetikzlibrary{bayesnet}

\newcommand{\R}[1]{\mathbb{R}^{#1}}

\newtheorem{theorem}{Theorem}

\newtheorem{proposition}[theorem]{Proposition}
\theoremstyle{definition}

\theoremstyle{remark}

\title{Structured Neural Networks for Density Estimation and Causal Inference}

\author{%
  Asic Q. Chen$^{1}\thanks{Equal Contribution  $^\dag$Equal Senior Authorship}$ \quad Ruian Shi$^{1*}$\quad Xiang Gao$^{1}$ \quad Ricardo Baptista$^{2\dag}$\quad Rahul G. Krishnan$^{1\dag}$\\
  $^1$University of Toronto, Vector Institute \quad $^2$California Institute of Technology\\
  \texttt{\{asicchen, ruiashi, xgao, rahulgk\}@cs.toronto.edu} \\
   \texttt{rsb@caltech.edu} \\
}

\begin{document}

\maketitle

\begin{abstract}
Injecting structure into neural networks enables learning functions that satisfy invariances with respect to subsets of inputs. For instance, when learning generative models using neural networks, it is advantageous to encode the conditional independence structure of observed variables, often in the form of Bayesian networks. We propose the Structured Neural Network (StrNN), which injects structure through masking pathways in a neural network. The masks are designed via a novel relationship we explore between neural network architectures and binary matrix factorization, to ensure that the desired independencies are respected. We devise and study practical algorithms for this otherwise NP-hard design problem based on novel objectives that control the model architecture. We demonstrate the utility of StrNN in three applications: (1) binary and Gaussian density estimation with StrNN, (2) real-valued density estimation with Structured Autoregressive Flows (StrAFs) and Structured Continuous Normalizing Flows (StrCNF), and (3) interventional and counterfactual analysis with StrAFs for causal inference. Our work opens up new avenues for learning neural networks that enable data-efficient generative modeling and the use of normalizing flows for causal effect estimation. 

\end{abstract}

\doparttoc %
\faketableofcontents %

\newcommand{\ian}[1]{{\textcolor{orange}{#1}}}
\newcommand{\xgao}[1]{{\textcolor{purple}{#1}}}
\newcommand{\asic}[1]{{\textcolor{cyan}{#1}}}

\section{Introduction}
\label{sec:introduction}

\begin{wrapfigure}{r}{0.42\textwidth}
\vspace{-2.5em}
\centering
\includegraphics[width=\linewidth,trim={0 0.2cm 1cm 0.55cm}, clip]{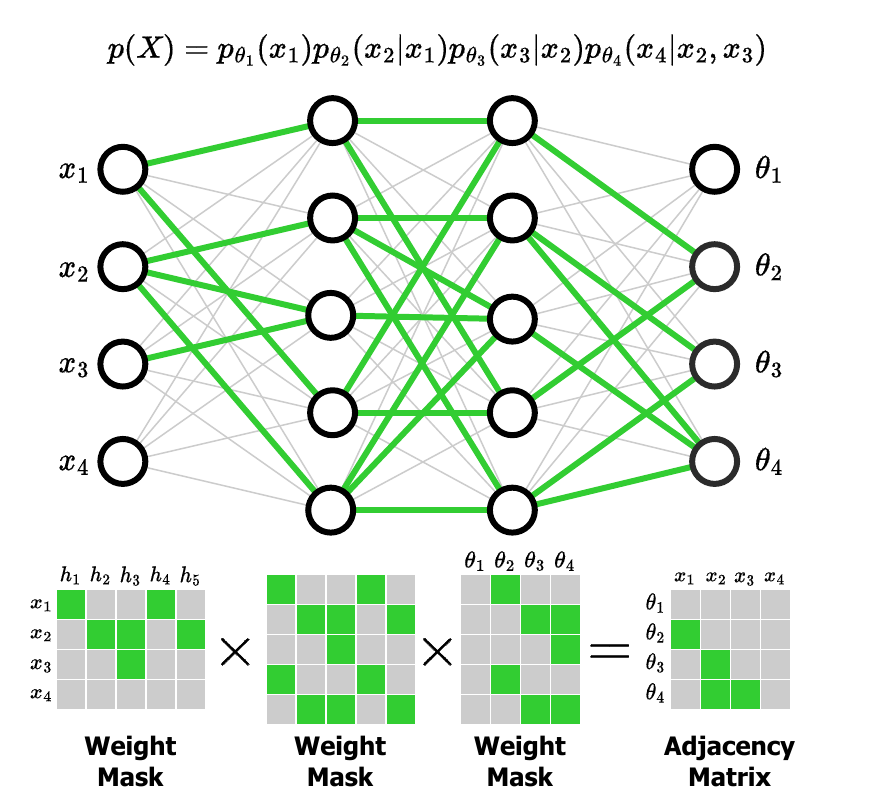} 
\caption{\small StrNN injects structure by masking the weights of a neural network. \textit{Top}: StrNN connections (green) compared to a fully connected network (gray). \textit{Bottom}: Binary factorization of an adjacency matrix yields weight masks. Masked weights shown in gray.\\}
\label{fig:method}
\vspace{-4.25em}
\end{wrapfigure}

The incorporation of structure into machine learning models has been shown to provide benefits for model generalization, learning efficiency, and interpretability. The improvements are particularly salient when learning from small amounts of data. This idea has found use in reinforcement learning \citep{ok2018exploration}, computational healthcare \citep{hussain2021neural, cui2020deterrent}, survival analysis \citep{gharari2023copula}, time series analysis \citep{curi2020structured}, and causal inference \citep{balazadeh2022partial}. 

This work focuses on the problem of density estimation from high-dimensional data which has been approached through a variety of lenses. For example, normalizing flows \citep{tabak2013family, rezende2015variational} model data by transforming a base distribution through a series of invertible transformations. Masked autoencoders (MADE) \citep{germain2015made} model the joint distribution via an autoregressive factorization of the random variables. The factorization is enforced by integrating structure in the neural network of the autoencoder. The MADE architecture zeros out weights in a neural network to ensure each output dimension has an autoregressive dependence on the input dimensions. 
When data is scarce, this may lead to over-fitting and harm generalization. When knowledge of a Bayesian network~\citep{pearl2011bayesian} and the associated conditional independencies exist, it is desirable to inject this knowledge directly into the network to improve density estimation. In this work, we extend the concept of weight masking, as in MADEs, to go beyond autoregressive dependencies of the output on the input. 

This work proposes the \textbf{Structured Neural Network (StrNN)}, a network architecture that enforces functional independence relationships between inputs and outputs via weight masking. In other words, the output can remain unaffected by changes to (subsets of) the input. We focus on instantiating this idea to model conditional independence between inputs when neural networks are deployed to model the density of random variables, as illustrated in Figure~\ref{fig:method}. Any set of conditional independence statements (e.g. in a Bayesian network) may be represented via a binary adjacency matrix. StrNN performs binary matrix factorization to generate a set of weight masks that follow the adjacency matrix. There are two key challenges we overcome. First, the general problem of binary matrix factorization in this context is under-specified, as there exist many valid masks whose matrix product realizes a given adjacency matrix. To this end, we propose the idea of neural network path maximization as a strategy to guide the generation of optimal masks. Secondly, binary matrix factorization is NP-hard in general. We study practical solutions that generate valid mask matrices efficiently. 

StrNN can then be applied to NN-based density estimation in various contexts. Where conditional independence properties are known a priori, we show that StrNN can be used to estimate parameters of data distributions while keeping specified variables conditionally independent. We further integrate StrNNs into various discrete and continuous flow architectures, including the autoregressive normalizing flow \citep{papamakarios2017masked, huang2018neural} to form the \textbf{Structured Autoregressive Flow (StrAF)}. The StrAF model uses the StrNN as a normalizing flow conditioner network, thus enforcing a given adjacency structure within each flow layer. The StrAF preserves variable orderings between chained layers, allowing the adjacency structure to be respected throughout the entire flow.

Finally, we study a natural application of StrAF in causal effect estimation, a domain that often requires flexible methods for density estimation that enforce conditional independence relationships within causal graphs~\citep{pearl2009causality}. We show how StrAFs can be used to perform interventional and counterfactual queries better than existing flow-based causal models that do not incorporate graphical structures. 
Across the board, we highlight how incorporating conditional independence structure improves generalization error when learning from a small number of samples.

In summary, the main contributions of this work are as follows:
\begin{enumerate}[leftmargin=*]

    \item We introduce StrNN, a weight-masked neural network that can efficiently learn functions with specific variable dependence structures. In particular, it can inject prior domain knowledge in the form of Bayesian networks to probability distributions. We formalize the weight masking as an optimization problem, where we can pick the objective based on desired neural architectures. We propose an efficient binary matrix factorization algorithm to mask arbitrary neural networks.
    \item We integrate StrNN into autoregressive and continuous normalizing flows for best-in-class performance in density estimation and sample generation tasks.
    \item We apply StrAF for causal effect estimation and showcase its ability to outperform existing causal flow models in accurately addressing interventional and counterfactual queries.
\end{enumerate}

\section{Background}
\label{sec:background}
\textbf{Masked Autoencoders for Density Estimation (MADE):} Masked neural networks were introduced for density estimation on binary-valued data \citep{germain2015made}. Given $\mathbf{x}=(x_1, ..., x_d)$, MADE factorizes $p(\mathbf{x})$ as the product of the outputs of a neural network. Writing the $j$-th output as the conditional probability $\hat{x}_j := p(x_j=1|\mathbf{x}_{<j})$, the joint distribution can be rewritten exactly as the binary cross-entropy loss.
As long as the neural network outputs are autoregressive in relation to its inputs, we can minimize the cross-entropy loss for density estimation. To enforce the autoregressive property for a neural network $y = f(x)$ with a single hidden layer and \textit{d} inputs and outputs, MADE element-wise multiplies the weight matrices W and V with binary masks $M^W$ and $M^V$:
\begin{equation}
\label{eq:NN_masked}
h(x) = g((W \odot M^W)x + b), \quad y = f((V \odot M^V)h(x) + c).
\end{equation}
The  autoregressive property is satisfied as long as the product of the masks, $M^VM^W \in \mathbb{R}^{d \times d}$, is lower triangular. The MADE masking algorithm~\eqref{alg:MADE} can be extended to neural networks with an arbitrary number of hidden layers and hidden sizes. For Gaussian data, the MADE model can be extended as $\mathbb{R}^d \rightarrow \mathbb{R}^{2d}$, $(x_1, ...x_d) \rightarrow (\hat{\mu}_1, ..., \hat{\mu}_d, \log(\hat{\sigma}_1), ..., \log(\hat{\sigma}_d))$. The last mask must be duplicated to ensure $\mu_j$ and $\sigma_j$ only depend on $\mathbf{x}_{<j}$. MADE can also be used as the conditioner in an autoregressive flow to model general data, as seen in~\citet{papamakarios2017masked}.

\textbf{Normalizing Flows:}
Normalizing flows~\citep{rezende2015variational} model complex data distributions and have been applied in many scenarios \citep{papamakarios2021normalizing}. Given 
data $\mathbf{x}\in \mathbb{R}^d$, a normalizing flow $\mathbf{T}\colon \mathbb{R}^d \rightarrow \mathbb{R}^d$ takes $\mathbf{x}$ to latent variables $\mathbf{z}\in \mathbb{R}^d$ that are distributed according to a simple base distribution $p_{\mathbf{z}}$, such as the standard normal. The transformation $\mathbf{T}$ must be a diffeomorphism (i.e., differentiable and invertible) so that we can compute the density of $\mathbf{x}$ via the change-of-variables formula:
    $p_{\mathbf{x}}(\mathbf{x}) = p_\mathbf{z}(\mathbf{T}(\mathbf{x}))|\det J_{\mathbf{T}}(\mathbf{x})|.$
We can compose multiple diffeomorphic transformations $\mathbf{T}_k$ to form the flow $\mathbf{T} = \mathbf{T}_1 \circ \dots \circ \mathbf{T}_K$ since diffeomorphisms are closed under composition. %
The flows are trained by maximizing the log-likelihood of the observed data under the density $p_{\mathbf{x}}(\mathbf{x})$. The log-likelihood can be evaluated efficiently when it is tractable to compute the Jacobian determinant of $\mathbf{T}$; for example when $\mathbf{T}_k$ is a lower triangular function~\citep{marzouk2016introduction}. Given the map, we can easily generate i.i.d.\thinspace samples from the learned distribution by sampling from the base distribution $\mathbf{z}^i \sim p_{\mathbf{z}}$ and evaluating the flow $\mathbf{T}^{-1}(\mathbf{z}^i)$. 

\textbf{Density Estimation with Autoregressive Flows:} 
When the Jacobian matrix of each flow layer is lower triangular, its determinant is simply the product of its diagonal entries \citep{huang2018neural}. 
This gives rise to the \textbf{autoregressive flow} formulation: given an ordering $\pi$ of the $d$ variables in the data vector $\mathbf{x}$, the $j$th component of the flow $\mathbf{T}$ has the form:
$x_j = \tau_j(z_j; c_j(\mathbf{x}_{<\pi(j)}))$ %
where each $\tau_j$ is an invertible \textit{transformer} and each $c_j$ is a \textit{conditioner} that only depends on the variables that come before $x_j$ in the ordering $\pi$. As a result, the map components define an autoregressive model that factors the density over a random variable $x$ as:
$p(\mathbf{x}) = \prod ^d _{j=1}p(x_j|\mathbf{x}_{<j})$ where 
$\mathbf{x}_{<j} = (x_1,\dots,x_{j-1}).$
\textit{autoregressive}. 
Under mild conditions, any arbitrary distribution $p_{\mathbf{x}}$ can be transformed into a base distribution with a lower triangular Jacobian~\citep{rezende2015variational}. That is, autoregressive flows are arbitrarily expressive given the target distribution.
One common choice of invertible functions for the transformer are %
monotonic neural networks~\citep{NEURIPS2019_2a084e55}.

\textbf{Density Estimation with Continuous Normalizing Flows:}
Continuous normalizing flows (CNFs) \citep{chen2018neural, grathwohl2018ffjord} represent the transformation $\mathbf{T}$ as the flow map solving the differential equation $\frac{\partial\mathbf{z}(t)}{\partial t} = f(\mathbf{z}(t), t; \theta)$. %
Given the initial condition $\mathbf{x} = \mathbf{z}(t_1) \sim p_{\mathbf{x}}$, we can integrate $f$ backwards in time from $t_1$ to $t_0$ to obtain $\mathbf{z} = \mathbf{z}(t_0) \sim p_{\mathbf{z}}$, or vice versa. In order to learn the CNF, \cite{chen2018neural} computes the change in log-density under the transformation using the \textit{instantaneous change of variables} formula, which is defined by the differential equation
$\frac{\partial \log p(\mathbf{z}(t))}{\partial t} = -\text{Tr}(\frac{\partial f}{\partial \mathbf{z}(t)})$. %
This expression is used to compute the log-likelihood of a target sample as $\log p(\mathbf{z}(t_1)) = \log p (\mathbf{z}(t_0)) - \int_{t_0}^{t_1} \text{Tr}(\frac{\partial f}{\partial \mathbf{z}(t)}) dt$. We refer the reader to~\cite{chen2018neural} for
the process of back-propagating through the objective using the adjoint sensitivity method as well as a discussion on the existence and uniqueness of a solution for the ODE flow map.

\textbf{Causal Inference with Autoregressive Flows:}
Modelling causal relationships is crucial for enabling effective decision-making in various fields \citep{pearl2009causality}. A structural equation model (SEM) parameterizes the process that generates observed data, allowing us to reason about interventions and counterfactuals. Given random variables $\mathbf{x}=(x_1, ..., x_d) \in \mathbb{R}^d$ with joint distribution $\mathbb{P}_\mathbf{x}$, the associated SEM consists of $d$ structural equations of the form $x_j = f_j(\mathbf{pa}_j, u_j)$, where $u$ represents mutually independent latent variables and $\mathbf{pa}_j$ denotes the direct causal parents of variable $x_j$. Each SEM also corresponds to a directed acyclic graph (DAG), with a causal ordering $\pi$ defined by the variables' dependencies.

\citet{khemakhem2021causal} showcased the intrinsic connection between SEMs and autoregressive flows. The authors demonstrated that affine autoregressive flows with a fixed ordering of variables can be used to parameterize SEMs under the framework of \textit{causal autoregressive flows} (CAREFL). When the causal ordering of variables is given, CAREFL outperforms other baselines in interventional tasks and generates accurate counterfactual samples. However, one significant limitation of CAREFL is that there is no guarantee that the autoregressive structure corresponds to the true dependencies in the causal graph beyond pairwise examples. In Section \ref{sec:structured_CAREFL}, we leverage StrAF to incorporate additional independence structure, enhancing its faithfulness to the causal DAG.

\section{Methodology}

\subsection{Structured Neural Networks}\label{sec:StrNN}
As neural networks are universal function approximators~\citep{hornik1989multilayer}, injecting structure into multilayer perceptrons (MLPs) naturally allows us to approximate complex functions with certain invariances. For a neural network that approximates an arbitrary function $f: \mathbb{R}^m \rightarrow \mathbb{R}^n$, one type of invariance we could consider is when a specific output $\hat{x}_j$ is independent from a given input $x_i$, i.e.: $\hat{x}_j \perp x_i$ means $\frac{\partial \hat{x}_j}{\partial x_i} = 0$. As a motivating example, we focus on the probabilistic density estimation problem framed as learning maps from one probability distribution to another. In this setting, it is imperative that we are able to learn structured functions between two high-dimensional spaces. Therefore in this paper, we seek to efficiently encode such invariances via weight masking.

For data $\mathbf{x}=(x_1, ..., x_d)$, we use the lower-triangular adjacency matrix $A \in \{0, 1\}^{d \times d}$ to represent the underlying variable dependence structure. In other words, $A_{ij} = 0$ for $j < i$ if and only if $x_i \perp x_j | x_{\{1,\dots,i\} \setminus j}$ and $A_{ij} = 1$ otherwise. This matrix encodes the same information as a Bayesian network DAG of the variables. In the fully autoregressive case, matrix A is a dense lower triangular matrix with all ones under the diagonal, which is the only case addressed in~\cite{germain2015made}. Their proposed MADE algorithm (Appendix \ref{appendix:MADE}) only encodes the structure of dense adjacency matrices, and cannot incorporate additional conditional independencies. Further, the non-deterministic version of the algorithm can introduce unwanted independence statements by chance, as discussed in Proposition~\ref{prop:randomMADE}.

We improve upon the idea of masked autoregressive neural networks to directly encode the independence structure represented by an adjacency matrix $A$ that is lower triangular but also has added sparsity. We observe that a masked neural network satisfies the structural constraints prescribed in $A$ if the product of the masks for each hidden layer has the same locations of zero and non-zero entries as $A$. Therefore, given the conditional independence structure of the underlying data generating process, we can encode structure into an autoregressive neural network by \textit{factoring the adjacency matrix into binary mask matrices for each hidden layer}.

More specifically, given an adjacency matrix $A \in \{0, 1\}^{d \times d}$ and a neural network with $L$ hidden layers, each with $h_1, h_2, ..., h_L$ hidden units ($\ge d$), we seek mask matrices $M^1 \in \{0, 1\}^{h_1 \times d}, M^2 \in \{0, 1\}^{h_2 \times h_1}, \dots, M^L \in \{0, 1\}^{d \times h_L}$ such that
$\label{eq:strNN_problem}
    A' \sim A, \text{where } A' := M^{L} \cdot \cdot \cdot M^2 \cdot M^1$.
We use $A'\sim A$ to denote that matrices $A'$ and $A$ share the same sparsity pattern, i.e.: exact same locations of zeros and non-zeros. Note that here $A$ is a binary matrix and $A'$ is an integer matrix. We then mask the neural network's hidden layers using $M^1, M^2, ..., M^L$ as per~\eqref{eq:NN_masked} to obtain a \textbf{Structured Neural Network (StrNN)}, which respects the prescribed independence constraints. The value of each entry $A'_{ij}$ thus corresponds to the number connections flowing from input $x_j$ to output $\hat{x}_i$ in the StrNN.

Finding the optimal solution to this problem is NP-hard since binary matrix factorization can be reduced to the biclique covering problem~\citep{miettinen2020recent, ravanbakhsh2016boolean, ORLIN1977406}. Furthermore, most existing works focus on deconstructing a given matrix $A$ into low-rank factors while minimizing (but not eliminating) reconstruction error~\citep{dan2015low, fomin2020parameterized}. In our application, any non-zero reconstruction error breaks the independence structure we want to enforce in our masked neural network. This puts existing algorithms for low-rank binary matrix factorization outside the scope of our paper. We instead consider the problem of finding factors that reproduce the adjacency matrix exactly, which is always possible when hidden layer dimensions are greater than the input and output dimension.

\textbf{Optimization Objectives.} Identifiability remains an issue even when we eliminate reconstruction error. Given an adjacency matrix $A$, there can be multiple solutions for factoring $A$ into per-layer masks that satisfy the constraints, especially if the dimensions of the hidden layers are much larger than $d$. Since the masks dictate which connections remain in the neural network, the chosen mask factorization algorithm directly impacts the neural network architecture. Hence, it is natural to explicitly specify a relevant objective to the neural network's approximation error during the matrix factorization step, such as the test log-likelihood.

Given that the approximation error is inaccessible when selecting the architecture, we were inspired by the Lottery Ticket Hypothesis~\citep{frankle2018lottery} and other pruning strategies~\citep{srivastava2014dropout, zoubin2016dropout} that identify a subset of valuable model connections. Our hypothesis is that given the same data and prior knowledge on independence structure, the masked neural network with \emph{more connections} is more expressive, and will thus be able to learn the data better and/or more quickly. 
To find such models, we consider two objectives: Equation~\eqref{eq: objective} that maximizes the number of connections in the neural network while respecting the conditional independence statements dictated by the adjacency matrix, and Equation~\eqref{eq: objective_variance} that maximizes connections while penalizing any pair of variables from having too many connections at the cost of the others. That is,

\vspace{-0.5em}
\hspace{0.075\textwidth}
\begin{minipage}[b]{.35\textwidth}
\begin{equation}
\label{eq: objective}
\max_{A' \sim A} \sum_{i=1}^d \sum_{j=1}^{d} A'_{ij},
\end{equation}
\end{minipage}
\hspace{.1\textwidth}
\begin{minipage}[b]{.35\textwidth}
\begin{equation}
\label{eq: objective_variance}
\max_{A' \sim A} \sum_{i=1}^d \sum_{j=1}^{i} A'_{ij} - \text{var}(A'),
\end{equation}
\end{minipage}
\hspace{0.1\textwidth}

where $\text{var}(A')$ is the variance across all entries in $A'$. While we focus on these two objectives, future work will find optimal architectures by identifying other objectives to improve approximation error.

\textbf{Factorization Algorithms.} The maximization of the aforementioned mask factorization is an intractable optimization problem. We therefore develop approximate algorithms to solve them. Our strategy for optimization involves recursively factorizing the mask matrix layer by layer. Given $A\in \{0, 1\}^{d \times d}$, we run Algorithm 1 once to find $ A_1 \in \{0, 1\}^{d \times h_1}$ and $ M^1 \in \{0, 1\}^{h_1 \times d}$ for a layer of width $h_1$ such that $A_1\cdot M^1 \sim A$, where $\sim$ denotes that the matrices share the same sparsity. For the next layer with $h_2$ hidden units, we use $A_1$ in place of $A$ to find $A_2 \in \{0, 1\}^{d \times h_2}$ and $M^2 \in \{0, 1\}^{h_2 \times h_1}$ such that $A_2 \cdot M^2 \sim A_1$. We repeat until we have found all the masks.

For each objective, we can obtain per-layer exact solutions using integer programming. While the Gurobi optimizer~\citep{gurobi} can be used for small $d$, this approach was found to be too computationally expensive for $d$ greater than 20, which is a severe limitation for real-world datasets and models. We hereby propose a greedy algorithm (shown in \textbf{Algorithm \ref{alg:mf}}) that \textit{approximates} the solution to the maximum connections objective in Equation~\eqref{eq: objective}. For each layer, the algorithm first replicates the structure given in the adjacency matrix $A$ by copying its rows into the first mask. It then maximizes the number of neural network connections by filling in the second mask with as many ones as possible while respecting the sparsity in $A$. See Appendix~\ref{appendix:greedy_algo} for a visual explanation of the algorithm. For a network with $d$-dimensional inputs and outputs and one hidden layer with $h$ units, this algorithm runs in $\mathcal{O}(dh)$ time. Scaling up to $L$ layers, where each hidden layer commonly contains $\mathcal O(d)$ units, the overall runtime is $\mathcal{O}(d^2L)$, which is much more efficient than the integer programming solutions. From our experiments, the greedy algorithm executes nearly instantaneously for dimensions in the thousands.

\begin{wrapfigure}{t}{0.44\textwidth}
    \small
    \centering
    \begin{minipage}{0.39\textwidth}
        \vspace{-1.5em}
        \centering
        \begin{algorithm}[H]
        \caption{Greedy factorization}\label{alg:mf}
        \KwData{$A\in \{0, 1\}^{d_1 \times d_2}$, hidden size h} 
        \KwResult{$M^V\in\{0, 1\}^{d_1 \times h}$, $M^W\in\{0, 1\}^{h \times d_2}$, satisfying $M^VM^W \sim A$}

        $nz \leftarrow$ non-zero rows in $A$\\
        Fill $M^W$ with $nz$; repeat until full\\
        Fill $M^V$ with ones\\
        for $i$-th row in $M^V$ do:\\
        \quad $C \leftarrow$ indices of 0's in $i$-th row of$A$\\
        \quad $T \leftarrow$ cols. of $M^W$ whose index $\in C$\\
        \quad $R \leftarrow$ indices of non-zero rows of $T$\\
        \quad for $r$ in $R$: set $M^V_{i, r}$ to zero\\
        return ($M^V$, $M^W$)\\
        \end{algorithm}
    \end{minipage}
    \vspace{-1em}
\end{wrapfigure}

In Appendix~\ref{appendix: binary_matrix_factorization}, we include detailed results %
from investigating the link between the neural network's generalization performance and the choice of mask factorization algorithm. We observe that while the exact solution to objective~\eqref{eq: objective} achieves a higher objective value than the greedy approach, it has no clear advantage in density estimation performance. Moreover, we found that models trained with the two objectives,~\eqref{eq: objective} and~\eqref{eq: objective_variance}, provide similar performance. However, some datasets might be more sensitive to the exact objective. For example, problems with anisotropic non-Gaussian structure may require neural architectures with more expressivity in some variables that may be favored with certain objectives. While we adopted objective~\eqref{eq: objective} and the efficient greedy algorithm in the remainder of our experiments, designing and comparing different factorization objectives is an important direction for future work.

\subsection{Structured Neural Networks for Normalizing Flows}
For the general real-valued probabilistic density estimation task, we highlight the use of StrNN in two popular normalizing flow frameworks, autoregressive flows and continuous-time flows. In the autoregressive flows setting, recall that the $j$th component of each flow layer is parameterized as $x_j = \tau_j(z_j; c_j(\mathbf{x}_{<\pi(j)}))$, where the conditioner $c_j$ dictates which inputs the latent variable can depend on. We use the StrNN as the conditioner, combined with any valid invertible transformer $\tau$, to form the \textbf{Structured Autoregressive Flow (StrAF)}. The StrNN conditioner ensures each transformed latent variable is only conditionally dependent on the subset of variables defined by a prescribed adjacency matrix. A single step from StrAF can be represented as: 
$\mathbf{z} = \text{StrAF}(\mathbf{x}) = \tau(\mathbf{x}, \text{StrNN}(\mathbf{x}, A))$ where $\mathbf{x} \in \mathbb{R}^d$ and $A \in \{0, 1\}^{d \times d}$ is an adjacency matrix. Normalizing flow steps are typically composed to improve the expressiveness of the overall flow. To ensure that the prescribed adjacency structure is respected throughout all flow steps, we simply avoid the common practice of variable order permutation between flow steps. We visualize this in Figure \ref{fig:gnf_comparison}. This choice does not hurt performance as long as a sufficiently expressive transformer is used. In our experiments, we apply the unconstrained monotonic neural network (UMNN) described in~\cite{NEURIPS2019_2a084e55} as a transformer.
 
\begin{wrapfigure}{r}{0.45\textwidth}
    \vspace{-2em}
    \centering
    \includegraphics[width=0.45\textwidth]{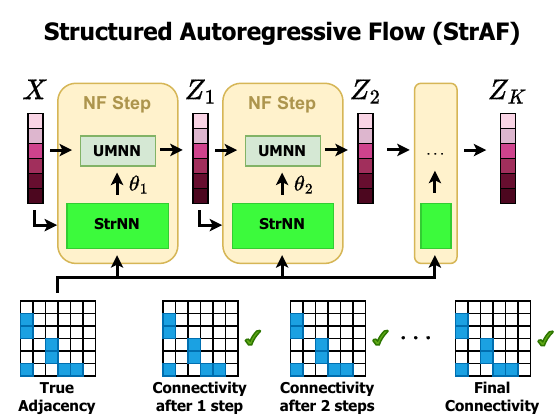}
    \caption{The StrAF injects a prescribed adjacency into each flow step using a StrNN conditioner. The StrAF does not permute latent variables, allowing the adjacency matrix to be respected throughout the entire flow.}
    \label{fig:gnf_comparison}
    \vspace{-1.5em}
\end{wrapfigure}

We also integrate the StrNN into continuous normalizing flows (CNFs), where the neural network that parameterizes $f$ in the  differential equation $\frac{\partial \mathbf{z}(t)}{\partial t} = f(\mathbf{z}(t), t)$ is replaced by a StrNN to obtain $\frac{\partial \mathbf{z}(t)}{\partial t} = \text{StrNN}(\mathbf{z}(t), t, A)$. This plug-in replacement allows us to inject structure into the function describing the continuous dynamics of the CNF without modifying other aspects of the CNF. We refer to this architecture as the \textbf{Structured CNF (StrCNF)}. The StrCNF uses the trace estimator described in FFJORD \citep{grathwohl2018ffjord} to evaluate the objective function for learning the flow.

While injecting structure, both StrAF and StrCNF inherit the efficiency of StrNN due to our choice of weight masking. Specifically, the output of the StrNN can be computed with a single forward pass through the network. In comparison, input masking approaches such as \cite{wehenkel2021graphical} must perform $d$ forward passes to compute the output for a single datum. This not only prevents efficient application of input masking to high dimensional data, but also is a barrier to integrating the method with certain architectures. For example, the CNF already requires many neural network evaluations to numerically solve the ODE defining the flow map, so making $d$ passes per evaluation is particularly inefficient.

\subsection{Structured Causal Autoregressive Flow}
\label{sec:structured_CAREFL}

Beyond density estimation, StrAF allows us to build autoregressive flows that faithfully represent variable dependencies defined by a causal DAG. In contrast, CAREFL~\citep{khemakhem2021causal} only maintains the autoregressive order, which is insufficient for data with more than two variables where the true structure must be characterized by a full adjacency matrix. Building on CAREFL, we also assume that the flow $\mathbf{T}$ takes on the following affine functional forms for observed data $\mathbf{x}$:
\begin{align}
\label{eq:affinef}
    x_j = e^{s_{j}(\mathbf{x}_{<\pi (j)})}z_{j} + t_j (\mathbf{x}_{<\pi (j)}), \quad j=1,...,d 
\end{align}
where the functions $s_j$ and $t_j$ are both conditioners that control the dependencies on the variables preceding $x_j$. It is crucial for the autoregressive or graphical structures to be maintained across all sub-transformations $\mathbf{T}_1, ..., \mathbf{T}_K$ for the flow $\mathbf{T} = \mathbf{T}_1 \circ \dots \circ \mathbf{T}_K$.

Assuming the true causal topology has been given either from domain experts or oracle discovery algorithms, StrAF can directly impose the given topological constraints by adding an additional masking step based on the adjacency matrix, as shown in section \ref{sec:StrNN}. This ensures that the dependencies of the flows match the known causal structure, which leads to more accurate inference predictions.

\section{Related Works} 
\label{sec:related}
Given a Bayesian network adjacency matrix, \citet{wehenkel2021graphical} introduced graphical conditioners to the autoregressive flows architecture through input masking. They demonstrated that unifying normalizing flows  with Bayesian networks showed promise in injecting domain knowledge while promoting interpretability, as even single-step graphical flows yielded competitive results in density estimation. Our work follows the same idea of introducing prior domain knowledge into autoregressive flows, but we instead use a masking scheme similar to methods in~\citet{germain2015made}. We again note that input masking scales poorly with data dimension $d$, as $d$ forward passes are required to obtain the output for a single datum.

\citet{EmbeddedModelFlows} proposed embedded-model flows, which alternates between traditional normalizing flows layers and gated structured layers that a) encode parent nodes based on the graphical model, and b) include a trainable parameter that determines how strongly the current node depends on its parent nodes, which alleviates error when the assumed graphical model is not entirely correct. In comparison, our work encodes conditional independence more directly in the masking step, improving accuracy when the assumption in the probabilistic graphical structure is strong. 

\citet{mouton2022graphical} applied a similar idea to residual flows by masking the residual blocks' weight matrices prior to the spectral normalization step according to the assumed Bayesian network. Similarly, \citet{weilbach20sccnf} introduced graphical structure to continuous normalizing flows by masking the weight matrices in the neural network that is used to parameterize the time derivative of the flow map. However, their method is difficult to apply to neural networks with more than a single hidden layer. In comparison, our factorization algorithm more naturally permits the use of multi-layered neural networks when representing CNF dynamics. The Zuko software package \citep{zuko} implements various types of normalizing flows, including MAFs that also rely on weight masking, but they only provide one possible algorithm to enforce autoregressive structure given a specific variable order. In comparison, our approach permits explicit optimization of different objectives during the adjacency matrix factorization step and we investigate the efficacy of factorization schemes and resulting neural architectures in our work. For completeness, we include pseudocode of their algorithm and comparisons to our own algorithms in Appendix \ref{appendix: binary_matrix_factorization} .

Flows have garnered increasing interest in the context of causal inference, with applications spanning various problem domains. \citet{ilse2021combining} parameterized causal model with normalizing flows in the general continuous setting to learn from combined observational and interventional data. \citet{melnychuk2022normalizing} used flows as a parametric method for estimating the density of potential outcomes from observational data. Flows have also been employed in causal discovery \citep{brouillard2020differentiable, khemakhem2021causal} as well as in various causal applications \citep{ding2023causalaf, wang2021harmonization}. In particular, \citet{balgi2022personalized} also considered embedding the true causal DAG in flows for interventional and counterfactual inference, but they do so via the framework of Graphical Normalizing Flows \citep{wehenkel2021graphical}. \citet{balgi2022counterfactual} used CAREFL on a real-world social sciences dataset leveraging a theorized Bayesian network.

\section{Experiments}
To demonstrate the efficacy of encoding structure into the learning process, we show that using StrNN its flow integrations to enforce a prescribed adjacency structure improves performance on density estimation and sample generation tasks. We experiment on both synthetic data generated from known structure equations and MNIST image data. Details on the data generation process for all synthetic experiments can be found in Appendix \ref{appendix: synthetic_data_gen}. We also apply StrAF in the context of causal inference and demonstrate that the additional graphical structure introduced by StrAF leads to more accurate interventional and counterfactual predictions. 
The code to reproduce these experiments is available at \url{https://github.com/rgklab/StructuredNNs}.

\subsection{Density Estimation on Binary Data}
\label{sec:experiment_binary_data}

\begin{figure}[h]
\includegraphics[width=\linewidth]{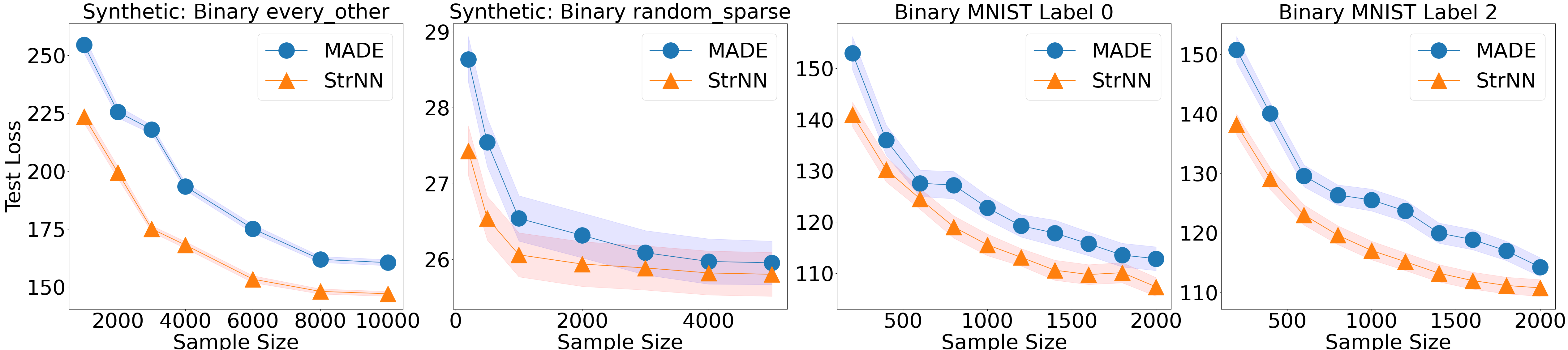} 
\caption{\small Negative log likelihood of a test set (lower is better) for density estimation experiments on binary synthetic data (left 2 images) and label-dependent binary MNIST data (right 2 images). Error ranges 
are reported as standard error across the test set. StrNN performs better than MADE as the sample size decreases.}
\label{fig:binary}
\end{figure}

For binary density estimation, we compare StrNN against the fully autoregressive MADE baseline. 

\textbf{Synthetic Tabular Data} We generate binary tabular data through structural equations from known Bayesian networks. The results are shown in Figure \ref{fig:binary} (left). We find that StrNN performs better than MADE, especially in the low data regime, as demonstrated on the left hand side of each chart. 

\textbf{MNIST Image Data} To study the effect of structure in image modeling, we use the binarized MNIST dataset considered in~\cite{germain2015made, Salakhutdinov2008OnTQ}. \cite{germain2015made} treated each 28-by-28-pixel image as a 784-dimensional data vector with full autoregressive dependence. Since we do not know the ground truth structure, we use StrNN to model a local autoregressive dependence on a square of a pixels determined by the hyper-parameter \texttt{nbr\_size}. By changing the hyperparameter we can increase the context window used to model each pixel. StrNN is equivalent to MADE for this experiment when we set \texttt{nbr\_size=28}. We first find the best \texttt{nbr\_size} for each label via grid search. For labels 0 and 2, the optimal \texttt{nbr\_size} is 10, and per-label density estimation results can be found in Figure~\ref{fig:binary} (right). StrNN outperforms MADE for both labels, with the advantage more significant when sample size is very small. Samples of handwritten digits generated from both StrNN and MADE can be found in Appendix \ref{appendix: additional_results}.

\begin{wrapfigure}{r}{0.5\textwidth}
\vspace{-2.4em}
\includegraphics[width=\linewidth]{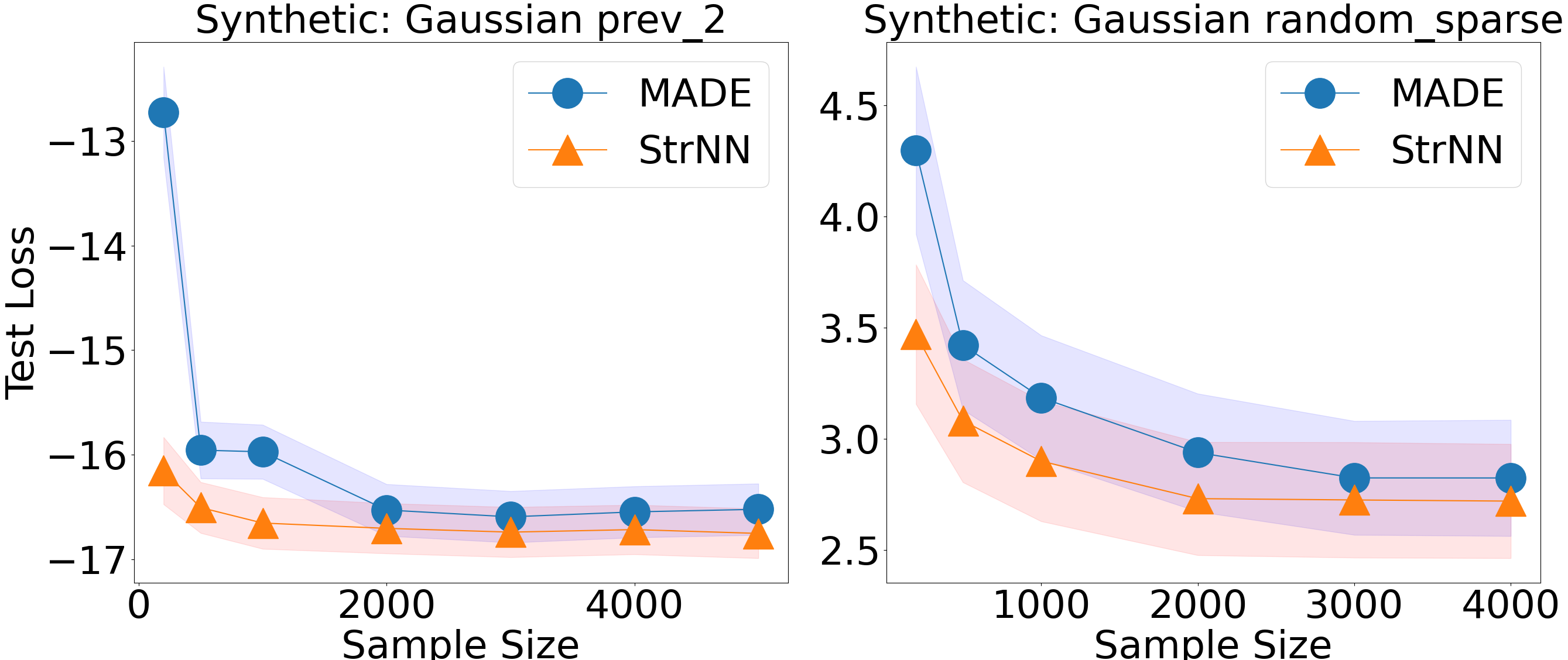} 
\caption{\small Results from density estimation experiments on Gaussian synthetic data generated from two different sparsity patterns. Test loss is reported in negative log likelihood with error ranges (standard error across test set). StrNN performs significantly better than MADE in the low data regime, and better on average.} 
\label{fig:gaussian_synth}
\vspace{-1.5em}
\end{wrapfigure}

\subsection{Density Estimation on Gaussian Data}
\label{sec:experiment_gaussian_data}

We run experiments to compare the performances of StrNN and MADE on synthetic Gaussian data generated from known structure equation models where each $x_i$ is Gaussian. 
We plot the results in Figure \ref{fig:gaussian_synth}. StrNN achieves lower test loss than MADE on average, although the error bars are not necessarily disjoint. When the sample size is low, however, StrNN significantly outperforms MADE, similar to the binary case. In conclusion, across all binary and Gaussian experiments, \textit{encoding structure} makes StrNN significantly more accurate at density estimation than the fully autoregressive MADE baseline.

\subsection{Density Estimation with Structured Normalizing Flows} \label{sec:straf_experiments}

We evaluate StrAF on density estimation against several baselines. We draw 1000 samples from a 15 dimensional tri-modal and non-linear synthetic dataset for experimental evaluation. Data generation is further described in Appendix \ref{appendix: flow_data_gen}.

\textbf{Experimental Setup} We select the fully autoregressive flow (ARF) and the Graphical Normalizing Flow (GNF) \citep{wehenkel2021graphical} as the most relevant discrete flow baselines for comparison. We use the official GNF code repository during evaluation, but note that it contains  design decisions that harm sample quality (see Appendix \ref{appendix: gnf}).  We also evaluate against FFJORD \citep{grathwohl2018ffjord} and \cite{weilbach20sccnf} as baselines for StrCNF. 
While other structured flows exist and have been examined in Section~\ref{sec:related}, they do not represent variables using an autoregressive structure and hence are less comparable.
Where applicable, models were provided the true adjacency matrix in their conditioners. All discrete flow models use a UMNN \citep{NEURIPS2019_2a084e55} transformer and we grid-search other hyperparameters as described in Appendix~\ref{appendix: nf_setup}. We evaluate density estimation performance using the negative log-likelihood (NLL). After fixing the hyperparameters, we perform 8 randomly initialized training runs, then report the mean and 95\% CI of the test NLL for these runs.

\begin{figure}[!h]
    \centering
    \includegraphics[width=\textwidth]{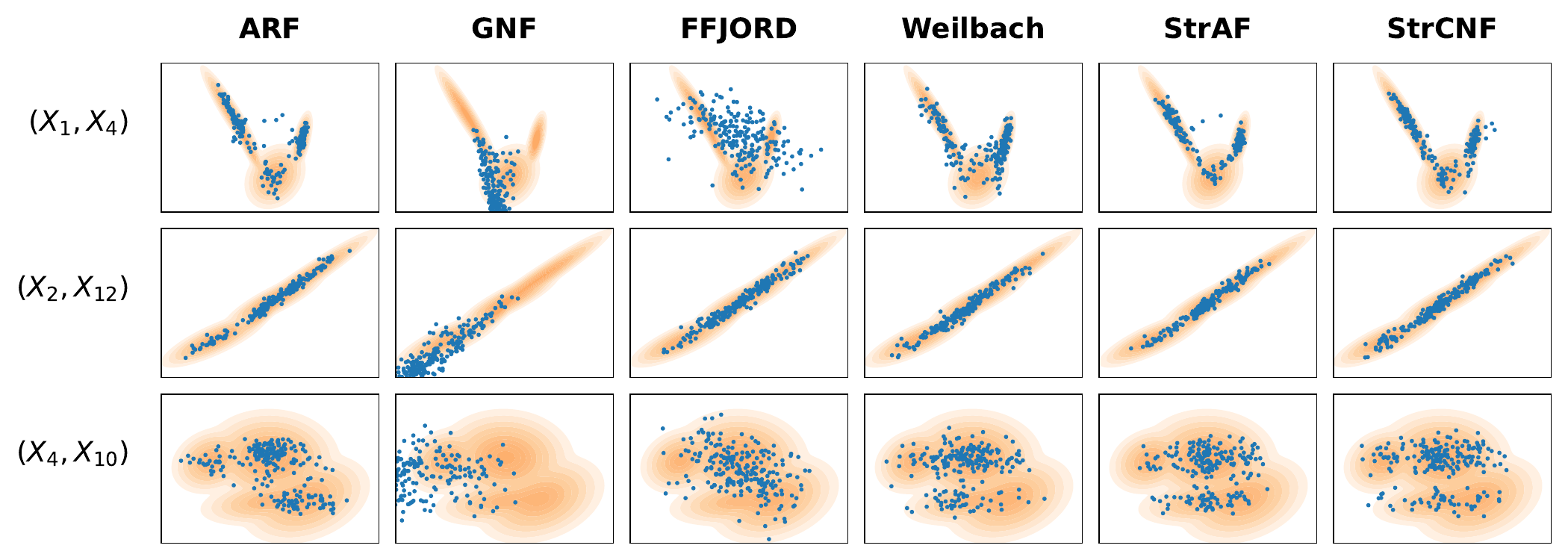}
    \caption{\small Model generated samples are shown in blue dots for randomly selected dimensions. The ground truth density is visualized by the orange contours. See Appendix \ref{appendix: gnf} for an explanation on why GNF performs poorly.}
    \label{fig:samples}
    \vspace{-0.75em}
\end{figure}

\textbf{StrNN improves flow-based models for density estimation}

We report results in Table \ref{tab:synthetic_comparison} and observe several trends. For both discrete and continuous flows, the ability to incorporate structure yields performance benefits compared to the ARF and FFJORD baselines. In particular, the StrNN offers advantages in comparison to baseline approaches that can encode structure. For example, while the method proposed by \cite{weilbach20sccnf} allows a DAG structure to \begin{wraptable}{r}{0.3\textwidth}
    \small
    \centering
    \caption{\small Evaluation of flow-based models. Mean and 95\% CI of test NLL over 8 runs reported.}
    \begin{tabular}{lcccc}
    \toprule
     & Test NLL ($\downarrow$) \\
     \midrule
    ARF \hspace{2em}& -3.09 $\pm$ 0.43 \\
    GNF & \textbf{-3.63 $\pm$ 0.35} \\
    StrAF & \textbf{-3.55 $\pm$ 0.20} \\
    \midrule
    FFJORD & -1.85 $\pm$ 0.64 \\
    Weilbach & -2.59 $\pm$ 0.58 \\
    StrCNF & \textbf{-4.01 $\pm$ 0.12} \\
    \bottomrule
    \end{tabular}
    \label{tab:synthetic_comparison}
    \vspace{-2em}
\end{wraptable}be injected into the baseline CNF, its inability to easily use a multi-layered neural network to represent dynamics hinders performance as compared to the StrCNF. We observe that StrAF and GNF perform comparably, and the StrCNF outperforms all other models.  We visualize the quality of samples generated by these flow models in Figure \ref{fig:samples}, and find that both StrAF and StrCNF yields samples that closely match the ground truth distribution.

\subsection{Causal Inference with Structured Autoregressive Flows}
\label{sec:experiment_causal}

We conduct synthetic experiments where we generate data according to a linear additive SEM. Details on data generation can be found in Appendix \ref{appendix: causal_data}. In our experiments involving 5- and 10-variable SEMs, we compare StrAF against CAREFL, which only utilizes MADE as the conditioner and maintains autoregressive ordering of the variables. This comparison highlights the additional benefits of enforcing the generative model to be faithful to the causal graph, which is a feature unique to StrAF. Furthermore, unlike the previous work conducted by \cite{khemakhem2021causal} that evaluates causal queries on individual variables alone, we propose a comprehensive evaluation metric called total mean squared error (MSE) for these two causal tasks as outlined below. We report the mean errors along with the standard deviations from multiple training runs with different datasets. The formulations of the metrics can be found in Appendix \ref{appendix: causal_metrics}. In addition, Appendix \ref{appendix: causal_algorithms} provides detailed algorithms and related discussions on generating interventional samples and computing counterfactuals with flows.

\begin{figure}[ht]
\begin{subfigure}{0.49\textwidth}
    \centering
    \includegraphics[width=\textwidth]{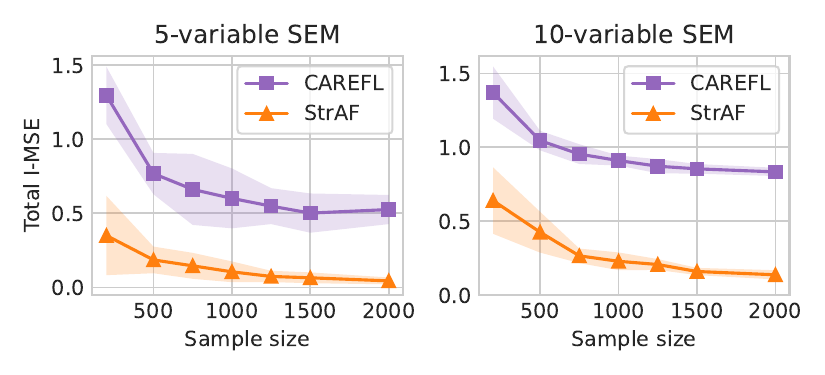}
    \vspace{-2em}
    \caption{Interventions}
    \label{fig:intervention_figure}
\end{subfigure}
\begin{subfigure}{0.49\textwidth}
    \centering
    \includegraphics[width=\textwidth]{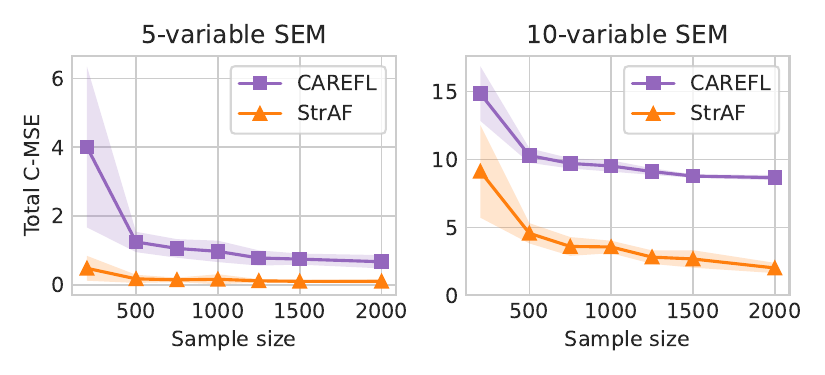}
    \vspace{-2em}
    \caption{Counterfactuals}
    \label{fig:counterfactual_figure}
\end{subfigure}
\caption{\small Evaluations of causal predictions (\textbf{left}: interventions; \textbf{right}: counterfactuals) on 5- and 10-variable SEMs made by StrAF and CAREFL. Performance is measured by the corresponding total mean squared error with standard deviation across ten runs. (a) measures the error of the expected value of one variable under different interventions, while (b) computes the error by deriving counterfactual values under different observed samples and queries.}
\end{figure}

\textbf{Interventions.} We perform interventions on each variable $x_j$ within the SEM and draw samples from the intervened causal system. Interventions are performed by setting the intervened value $\alpha$ using one of eight integers perturbed from the mean of the intervened variable. We compute the expectations $\mathbb{E}[x_i|do(x_j=\alpha)]$ for variable $x_i$, excluding the intervened variable $x_j$ itself and its preceding variables as they remain unaffected. To estimate these expectations, we use 1000 samples and calculate the squared error for the predictions. The resulting error is then averaged over the set of intervened values, intervened variables, and the corresponding variables with computed expectations under interventions. We present the plot of this aggregated metric, referred to as the total intervention mean squared error (total I-MSE), for the two SEMs in Figure \ref{fig:intervention_figure}. Notably, StrAF consistently outperforms CAREFL across all training dataset sizes, highlighting the effectiveness of the additional graph structure in improving StrAF's inference of the interventional distribution.

\textbf{Counterfactuals.} In a similar setting as the intervention experiments, we evaluate StrAF and CAREFL on their ability to compute accurate counterfactuals conditioned on observed data. Counterfactual inference tackles what-if scenarios: determining the value of variable $x_i$ if variable $x_j$ had taken a different value $\alpha$. Unlike interventions, counterfactual queries involve deriving the latent variables $\mathbf{z}$ given the observed data $\mathbf{x} = \mathbf{x}_{\text{obs}}$, rather than sampling new $\mathbf{z}$. We generate 1000 observations using the synthetic SEM and derive counterfactual values $\mathbf{x}$ by posing queries with varying $\alpha$ values for each variable $x_j$. Similar to interventions, we compute the squared error and average over the $1000$ observed samples and all possible combinations of counterfactual queries for each observed sample, and we refer to this metric as the total counterfactual mean squared error (total C-MSE). Figure \ref{fig:counterfactual_figure} illustrates that, for both SEMs, StrAF outperforms CAREFL in making more accurate counterfactual predictions. Moreover, StrAF demonstrates consistent performance even in scenarios with limited available samples.

\section{Conclusions and Limitations}

We introduce the \textbf{Structured Neural Network (StrNN)}, which enables us to encode functional invariances in arbitrary neural networks via weight masking during learning. For density estimation tasks where the true dependencies are expressed via Bayesian networks (or adjacency matrices), we show that StrNN outperforms a fully autoregressive MADE model on synthetic and MNIST data. We integrate StrNN in autoregressive and continuous flow models to improve both density estimation and sample quality. Finally, we show that our structured autoregressive flow-based causal model outperforms existing baselines on causal inference. We address some limitations and directions for future work below.

\textbf{Access to true adjacency structure.} We assume access to the true conditional independence structure. While this information is available in many contexts, such as from domain experts, there is also a wide literature on learning the conditional independencies directly from data~\citep{drton2017structure}. One prominent example is the NO-TEARS algorithm~\citep{zheng2018dags}. \citet{wehenkel2021graphical} shows that integrating NO-TEARS with an autoregressive flow can improve density estimation when a ground truth adjacency is unknown. Any adjacency matrix learned from data can also be integrated in StrAF. In our work, we have found two existing causal structure discovery libraries that are relatively comprehensive and easy to use: \cite{kalainathan2020causal} and \cite{causallearn}. Future work will use StrNN to directly learn structure from data, providing a full pipeline from structure discovery to density estimation and sample generation. 

\textbf{StrNN optimization objectives.}
The mask factorization algorithm used by StrNN can maximize different objectives while ensuring the matrices satisfy a sparsity constraint. In Section \ref{sec:StrNN}, we proposed two such objectives and in Appendix~\ref{appendix: binary_matrix_factorization} we demonstrated that they can impact model generalization. StrNN provides a framework with which it is possible to explore other objectives to impose desirable properties on neural network architectures. Moreover, investigating the effect of sparse structure on faster and easier training is a valuable direction~\citep{frankle2018lottery}. For example, one approach would be to leverage connections between dropout and the lottery ticket hypothesis to randomly introduce  sparsity into a neural network using StrNN weight masking.

\vspace{-0.5 em}
\acksection
\vspace{-0.8 em}
We thank David Duvenaud, Tom Ginsberg, Vahid Balazadeh-Meresht, Phil Fradkin, and Michael Cooper for insightful discussions and draft reviewing.
We thank anonymous reviewers for feedback that has greatly improved the work. This research was supported by an NSERC Discovery Award RGPIN-2022-04546, and a Canada CIFAR AI Chair. AC is supported by a DeepMind Fellowship and RS is supported by the Ontario Graduate Scholarship and the Ontario Institute for Cancer Research. Resources used in preparing this manuscript were provided in part, by the Province of Ontario, the Government of Canada through CIFAR, and companies sponsoring the Vector Institute.

\newpage
\bibliographystyle{unsrtnat}
\bibliography{refs}

\newpage 
\appendix
\addcontentsline{toc}{section}{Appendix} %
\part{Appendix} %
\parttoc %
The code used in this project can be found at:  \url{https://github.com/rgklab/StructuredNNs}.

\section{Mask Algorithms and Binary Matrix Factorization}
\label{appendix: binary_matrix_factorization}

\subsection{MADE Algorithm and Limitations}
\label{appendix:MADE}

As mentioned in the main text, the central idea of masking neural network weights to inject variable dependence was inspired by the work of Germain et al.~\cite{germain2015made}. In their paper, the authors proposed Algorithm~\ref{alg:MADE} to ensure the outputs of an autoencoder are autoregressive with respect to its inputs.

\begin{algorithm}[hbt!]
\caption{MADE Masking Algorithm }\label{alg:MADE}
    \SetKwInOut{Input}{Input}
    \SetKwInOut{Output}{Output}
    \Input{Dimension of inputs $d$, Number of hidden layers $L$, Number of hidden units $h$}
    \Output{Masks $M^1,\dots,M^{L+1}$} 

    \% Sample $\mathbf{m}^l$ vectors\;
    $\mathbf{m}^0 \leftarrow \text{shuffle}([1,\dots,d])$\;
    \For{$l=1$ to $L$}{
        \For{$k=1$ to $h^l$}{
           $\mathbf{m}^l(k) \leftarrow \text{Uniform}([\min(\mathbf{m}^{l-1}),\dots,d-1])$\;
        }
    }
    \% Construct masks matrices\;
    \For{$l = 1$ to $L$}{
        $M^{l} \rightarrow \mathds{1}_{\mathbf{m}^l \geq \mathbf{m}^{l-1}}$\;
    }
    $M^{L+1} \rightarrow \mathds{1}_{\mathbf{m}^0 > \mathbf{m}^L}$\;
\end{algorithm}

As a concrete example, let us consider the case of a single hidden layer network with $d$ inputs and $h$ hidden units. Here, Algorithm~\ref{alg:MADE} first defines a permutation $\mathbf{m}^0 \in \R{d}$ of the set $\{1,\dots,d\}$, and then independently samples each entry in the vector $\mathbf{m}^1 \in \R{h}$ with replacement from the uniform distribution over the integers from  $1$ to $d-1$. This assignment is used to define the two binary masks matrices $M^1$ of size $h \times d$ and $M^2$ of size $d \times h$. The matrix product of the resulting masks $M^2M^1 \in \mathbb{R}^{d \times d}$ provides the network's connectivity. In particular, the $(i,j)$ entry of $M^2M^1$ indicates the dependence of output $i$ on input $j$. 

There are several key limitations of the MADE algorithm:
\begin{enumerate}[leftmargin=*] \itemsep0pt
    \item As mentioned in Section~\ref{sec:StrNN}, the MADE algorithm can only mask neural networks such that they respect the autoregressive property. It is not capable of enforcing additional conditional independence statements as prescribed by an arbitrary Bayesian network. For a general probability distribution that does not satisfy any conditional independence properties, we expect each marginal conditional in the factorization of the density $p(x) = \prod_{k=1}^d p(x_k|x_{<k})$ to depend on all previous inputs. As a result, the matrix product $M^2M^1$ should be fully lower-triangular, meaning that output $k$ depends on all inputs $1,\dots,k-1$. If there is conditional independence structure, however, the MADE algorithm does not provide a mechanism to define mask matrices such that their matrix product is sparse and hence the corresponding MADE network enforces these constraints on the variable dependence. %
    \item The non-deterministic MADE masking algorithm presented in~\cite{germain2015made}, the resulting mask matrices are not always capable of representing any distribution. In particular, the random algorithm can yield some mask matrices where the lower-triangular part of their matrix product is arbitrarily sparse, i.e., there exists some $k < k'$ such that $(M^2M^1)_{k,k'} = 0$. As a result, the MADE network with these masks enforces additional conditional independencies that are not necessarily present in the underlying data distribution. Proposition~\ref{prop:randomMADE} formalizes this point.
\end{enumerate}

\begin{proposition} \label{prop:randomMADE}
    There is a non-zero probability that Algorithm~\ref{alg:MADE} will yield masks that enforce unwanted conditional independencies.
\end{proposition}
\begin{proof} For a single hidden layer ($L = 1$) neural network with $h$ units, there is a probability $1/d^h > 0$ of sampling $\mathbf{m}^1 = \mathbf{1}$, i.e., each entry is independently sampled to be $1$. This vector yields a mask matrix $M^{1}$ that only has one non-zero column of ones at the index $k$ where $\mathbf{m}^0(k) = 1$. As a result, the matrix product $M^2M^1$ also has only one non-zero column at index $k$, meaning that all outputs $k' > k$ depend only on $x_k$ and not on other input variables. Therefore, these mask matrices enforce the constraints $X_{k'} \ci X_{<k} | X_k$ for all $k'$. Equivalently, a distribution that does not satisfy these constraint can not be represented using this MADE network.
\end{proof}

In our work, we adopt a weight masking scheme by solving a binary matrix factorization problem that overcomes these limitations. Both the globally optimal and approximate solutions proposed in our paper are deterministic. Thus, we enforce all conditional independence properties exactly and ensure unwanted variable independence statement do not appear in our neural networks.

\subsection{Mask Factorization: Integer Programming and Greedy Algorithm}
\label{appendix:greedy_algo}

In Section~\ref{sec:StrNN}, we showed that finding the weight masks for each neural network layer is equivalent to factoring the adjacency matrix that represents the input and output connectivity of these layers. We can find exact solutions to this problem by solving the optimization problem given in~\eqref{eq: objective}. This problem can be formalized as the following integer programming problem:

\begin{align}
    \label{eq: IP_formulation}
    &\text{Inputs: } A \in \{0, 1\}^{d\times d}\\
    &\text{Outputs: } M^V\in \{0, 1\}^{d\times h}, M^W \in \{0, 1\}^{h\times d} \nonumber \\
    \max &\sum_{i=1}^{d}\sum_{j=1}^d \ v_i w_j \nonumber\\
    \text{ such that }& v_iw_j > 0 \text{ if } A_{ij}=1 \nonumber \\
    & v_iw_j =0 \text{ if } A_{ij}=0 \nonumber \\
    &\text{ where } M^V = \begin{pmatrix}
        v_1 \\ v_2 \\ ... \\ v_d
    \end{pmatrix} \text{ and } v_i \in \{0, 1\}^{1\times h} \nonumber \\
    &\text{ and } 
    M^W = \begin{pmatrix}
        w_1 & w_2 & ... & w_d
    \end{pmatrix} \text{ and } w_j \in \{0, 1\}^{h\times 1} \nonumber
\end{align}

To formulate a similar problem for the objective given in~\eqref{eq: objective_variance} instead, we simply replace the integer programming objective with 
\begin{align}
    \label{eq:IP_formulation_var}
    \max{(\sum_{i=1}^{d}\sum_{j=1}^d \ v_i w_j  - \text{Var}_{i, j}(v_i w_j))}.
\end{align}

We used the Gurobi optimizer~\cite{gurobi} to solve the above integer programming problems in our experiments, and found that exact solutions are found reasonably quickly for up to $d=20$. For dimensions larger than 20, however, directly solving the integer programming problem becomes prohibitively expensive even on computing clusters with multiple GPUs, so it is intractable to seek exact solutions to these problems for most real-world datasets. Therefore, in this work we propose Algorithm~\ref{alg:mf}, a greedy method that gives an approximate solution to the problem~\ref{eq: IP_formulation} very efficiently. Figure~\ref{fig:greedy_visualization} provides a visual example of the steps performed by Algorithm~\ref{alg:mf}.

\begin{figure}[h]
  \centering

  \begin{subfigure}{\textwidth}
    \centering
    \includegraphics[width=\textwidth]{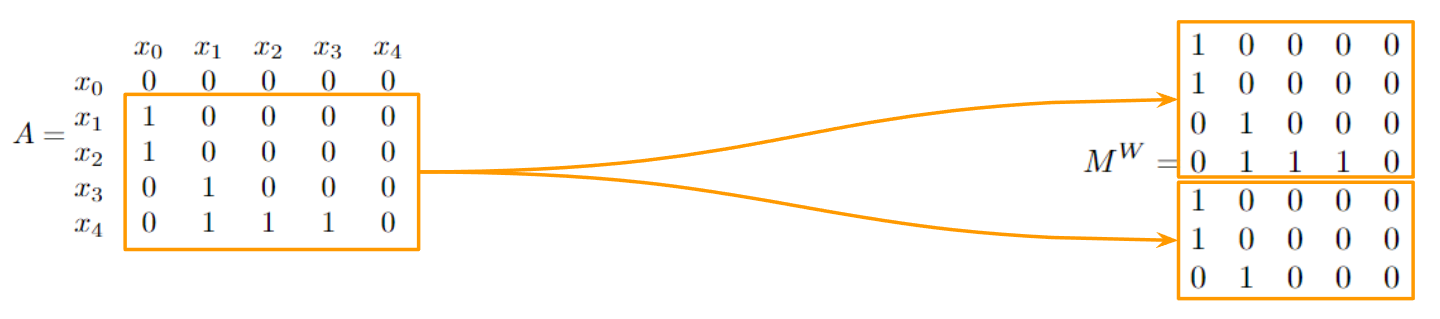}
    \caption{\small Step 1: Given the adjacency matrix $A$, we first populate the first layer mask $M^W$ by copying over non-zero rows in $A$, and repeating until all rows of $M^W$ are full.}
    \label{fig:step1}
  \end{subfigure}
\end{figure}
\begin{figure}
    \ContinuedFloat
  \begin{subfigure}{\textwidth}
    \centering
    \includegraphics[width=\textwidth]{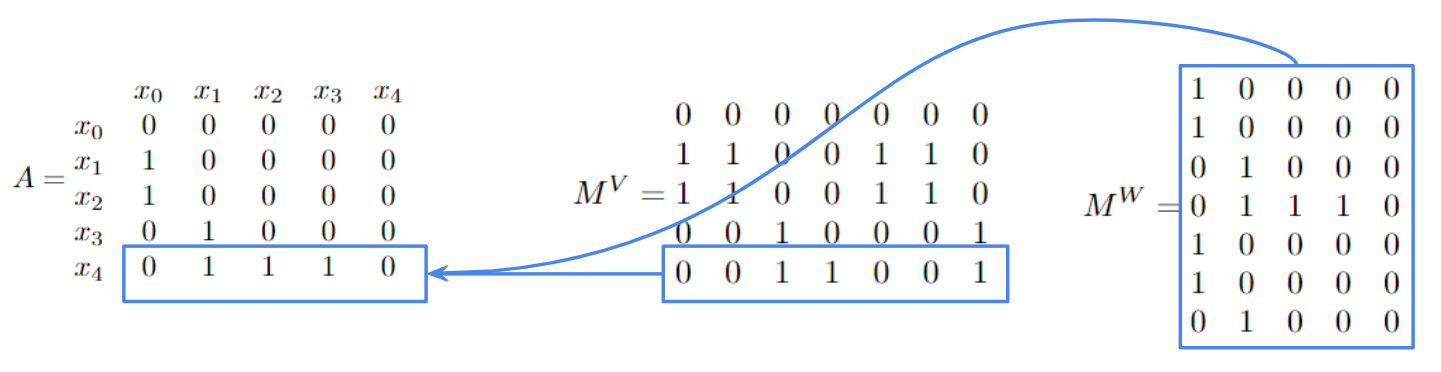}
    \caption{\small Step 2: Next, we populate the second layer mask $M^V$. Let us take the last row of $M^V$ for an example: to respect all conditional independence statements given by $A$, we need the product of the last row of $M^V$ and the $M^W$ matrix to have the same zero and non-zero locations as the last row of $A$. Since there are zeros in the first and last column of $A$'s last row, we need the products of the last row of $M^V$ with the first and last columns of $M^W$ to be zero.}
    \label{fig:step2}
  \end{subfigure}

    \begin{subfigure}{\textwidth}
    \centering
    \includegraphics[width=\textwidth]{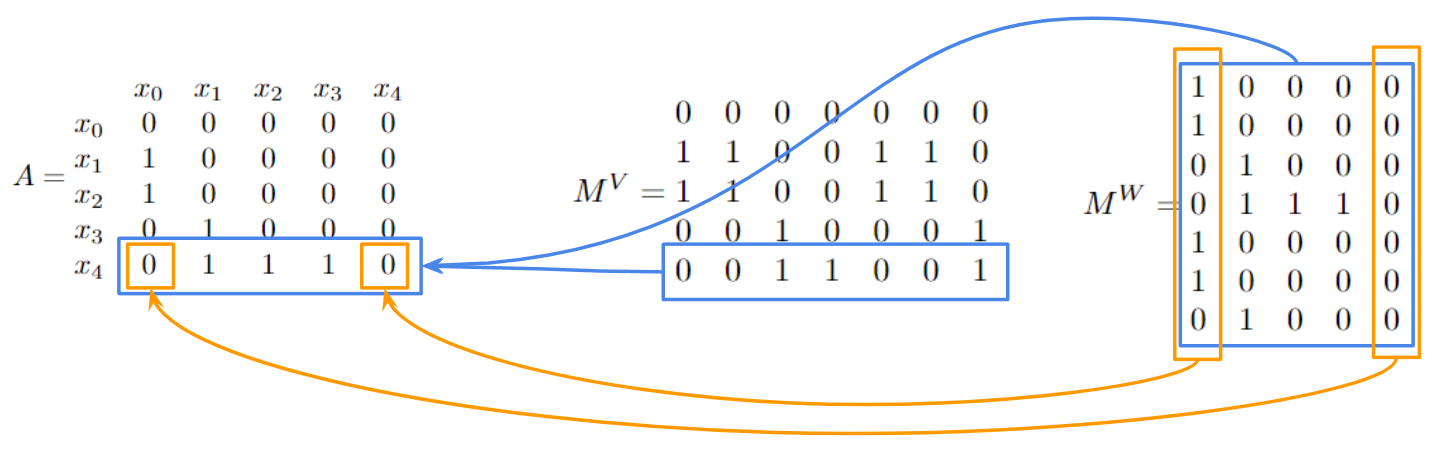}
    \caption{\small Step 3: We find the unique ones in the first and last columns of $M^W$ and set the corresponding positions in the last row of $M^V$ to zero.}
    \label{fig:step2-2}
  \end{subfigure}

    \begin{subfigure}{\textwidth}
    \centering
    \includegraphics[width=\textwidth]{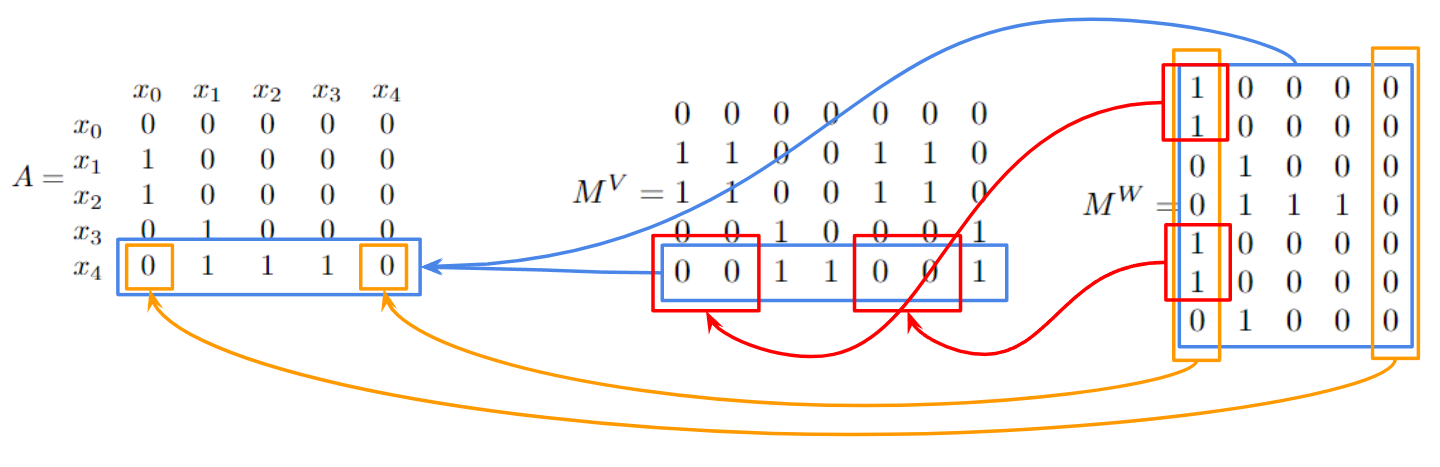}
    \caption{\small Step 4: Everything else in the last row of $M^V$ is set to $1$ to maximize the number of connections in the resulting neural network for the optimization objective in~\eqref{eq: objective}.}
    \label{fig:step2-3}
  \end{subfigure}

  \caption{\small A visual example of Algorithm 1 being applied on the adjacency matrix A for a neural network with $d=5$ inputs, $d=5$ outputs, and one single hidden layer containing $h=7$ hidden units.}
  \label{fig:greedy_visualization}
\end{figure}
\newpage
To estimate how well the greedy algorithm approximates the solution to problem~\ref{eq: IP_formulation}, we randomly sample lower triangular adjacency matrices, setting entries to 0 or 1 based on a given sparsity threshold  between 0 and 1. In other words, for the threshold $0.1$, the random adjacency matrix is very dense, and when the threshold is $0.9$, it is very sparse. For fixed input and output dimensions, we sample 10 such random adjacency matrices for each sparsity threshold ranging from $0.1$ to $0.9$, and take the average of the $\sum_{i=1}^{d}\sum_{j=1}^d \ v_i w_j$ objective value obtained by each factorization algorithm. Results for $d=10$-dimensional adjacency matrices are shown in Figure~\ref{fig:mask_connection_comparisons}. We see that the exact integer programming solution achieves higher objective values compared to the greedy algorithm we propose in Algorithm~\ref{alg:mf}, but it remains to further evaluate the resulting masks from both algorithms on their performance for a downstream density estimation task.

\begin{figure}[h]
  \centering

    \begin{subfigure}{\textwidth}
        \centering
        \includegraphics[width=\textwidth]{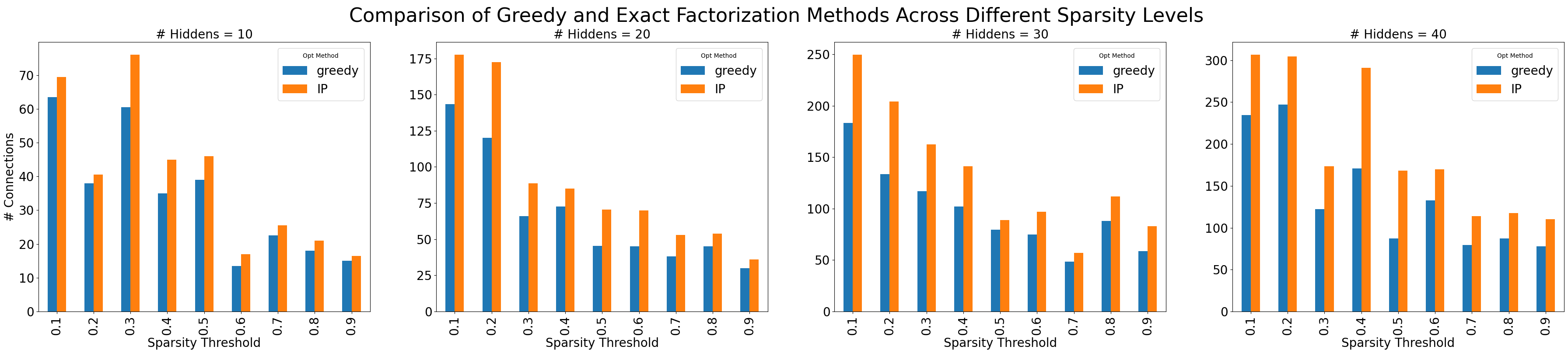}
        \caption{\small Randomly generated adjacency structures of 10 dimensions.}
    \end{subfigure}

    \begin{subfigure}{\textwidth}
        \centering
        \includegraphics[width=\textwidth]{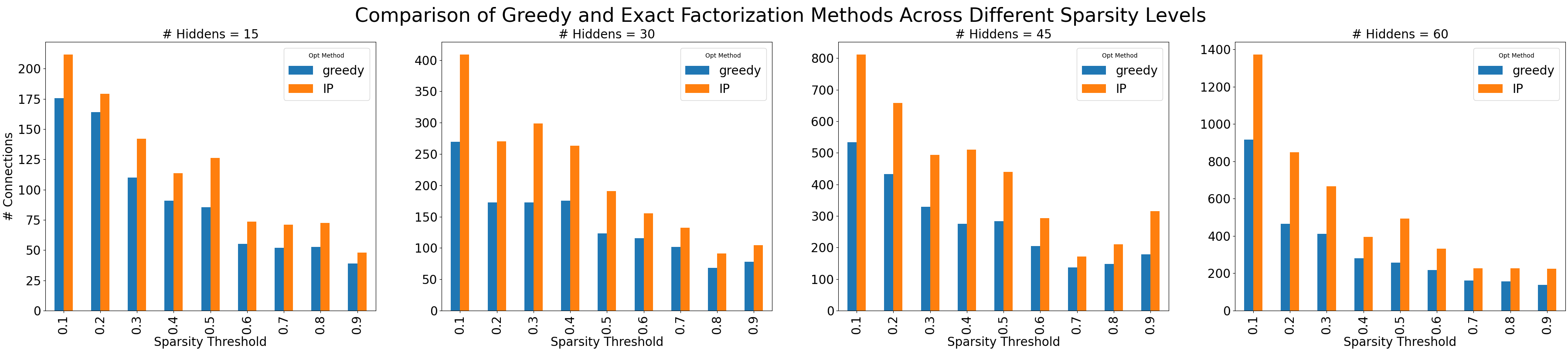}
        \caption{\small Randomly generated adjacency structures of 15 dimensions.}
    \end{subfigure}
    \caption{\small Comparing objective value (Equation~\ref{eq: objective}) achieved by greedy and exact integer programming (IP) methods. IP gives better objective values when the adjacency matrix is very sparse. As the number of neurons involved goes up, the difference in these methods also increases.}
    \label{fig:mask_connection_comparisons}
\end{figure}

\subsection{Mask Algorithms and Generalization}
\label{appendix:mask_general}

To that end, we evaluate the density estimation performance of masked neural networks on 20-dimensional synthetic binary data, using the following four mask factorization methods:
\begin{enumerate}[leftmargin=*]
    \item \textbf{MADE}: A fully autoregressive baseline using  Algorithm~\ref{alg:MADE} as proposed in~\cite{germain2015made} 
    \item \textbf{Greedy}: The proposed method in Algorithm~\ref{alg:mf}
    \item \textbf{IP}: The exact integer programming solution to Problem~\ref{eq: IP_formulation}
    \item \textbf{IP-var}: The exact integer programming solution to Problem~\ref{eq: IP_formulation}, but with the objective in~\eqref{eq:IP_formulation_var}
    \item \textbf{Zuko}: The approximate factorization algorithm proposed in~\citep{zuko} Github repository, as per our understanding. (See Algorithm~\ref{alg:zuko} for pseudocode and Appendix~\ref{appendix:zuko_comp} for a more detailed description of the Zuko method and its comparison to Algorithm~\ref{alg:mf}.)
\end{enumerate}

\begin{algorithm}[hbt!]
\caption{Zuko Masking Algorithm }\label{alg:zuko}
    \SetKwInOut{Input}{Input}
    \SetKwInOut{Output}{Output}
    \Input{A: $\{0,1\}^{d\times d}$: adjacency matrix, \\
    H = ($h_1$, ..., $h_{L-1}$)$\in \mathbb{N}^{L-1}$: list of $L-1$ hidden layer sizes}
    \Output{Masks $M_1 \in \{0,1\}^{h_1\times d}, M_2 \in \{0,1\}^{h_2 \times h_1}, \dots,M_{L}\in\{0,1\}^{d\times h_{L-1}}$,\\
    satisfying $M_L \times \dots \times M_2 \times M_1 \sim A$} 

    $A'$ $\leftarrow$ unique rows of $A$\;
    $inv$ $\leftarrow$ mapping of rows of $A'$ to the original row indices in $A$\;
    $P' \leftarrow A' \ \times \ A'.T$\;
    $n\_deps \leftarrow$ sums of rows of $A'$\;
    $P \leftarrow P' == n\_deps$\;
    \For{$i=1$ to $L-1$}{
        \eIf{$i = 0$}{
            $M_i \leftarrow A'$\;
        }{
            $M_i \leftarrow$ rows of $P$ denoted by $indices$\;
        }
        \eIf{$i < L-1$}{
            $reachable \leftarrow n\_deps \neq 0$\;
            $indices \leftarrow$ $reachable$ indices repeated to fill up to $h_i$\;
            $M_i \leftarrow M_i[indices]$\;
        }{
            $M_i \leftarrow M_i[inv]$
        }
    }
\end{algorithm}

The experiment setup and grid of hyperparameters used for these experiments are the same as those in all binary and Gaussian experiments in this work, as detailed in Appendix~\ref{sec:exp_setup_binary_gaussian}. Specifically, the adjacency structures used in these experiments are explained and visualized in~\ref{appendix: adjacency} and~\ref{appendix: binary_data_gen}. The results for the negative log-likelihood of a test dataset are reported in Table~\ref{tab:mask_algos_density_results}. We see that all three methods proposed and referenced in this work---Greedy, IP, IP-var, and Zuko---outperform the MADE baseline, but there is no clear winner based on overlapping error ranges. IP does not perform significantly better than Greedy based on the higher objective value achieved for~\eqref{eq: objective}. Meanwhile, there is no significant difference between the objective that maximizes the total connections (Equation~\ref{eq: objective} and the objective with the added variance penalty (Equation~\ref{eq: objective_variance}) when comparing the performance of IP versus IP-var. Hence, for efficiency and overall performance, we choose to adopt the Greedy mask factorization algorithm for the rest of the experiments in this paper.

\begin{table}[h]
\centering
{
\caption{\small Density estimation results on 20-dimensional synthetic datasets, reported using the negative log-likelihood on a held-out test set (lower is better). The error reported is the sample error across the test set. The four methods, Greedy, IP, IP-var, and Zuko perform better than the MADE baseline, but similarly to each other.}
  \begin{tabular}{lcccccc}
    \toprule
    \multirow{2}{*}{Dataset} &
      \multicolumn{2}{c}{Random Sparse} &
      \multicolumn{2}{c}{Previous 3} \\
     & {$n = 5000$} & {$n=2000$} & {$n=5000$} & {$n=2000$}\\
      \midrule
   MADE & $7.790 \pm 0.140$ & $7.788 \pm 0.142$ & $8.767 \pm 0.132$ & $8.816 \pm 0.134$ \\
   \midrule
   Greedy & $7.758\pm0.137$ & $7.778\pm0.142$ & $8.757\pm0.131$ & \textbf{8.768 $\pm$ 0.130} \\
   IP & $7.758\pm0.138$ & $7.769\pm 0.140$ & \textbf{8.755 $\pm$ 0.132} & $8.769\pm0.129$ \\
   IP-var & \textbf{7.757 $\pm$ 0.137} & \textbf{7.768 $\pm$ 0.140} & $8.758\pm 0.132$ & $8.770 \pm 0.131$\\
   Zuko & $7.758\pm0.137$ & $7.776\pm0.141$ & $8.757\pm0.130$ & \textbf{8.768 $\pm$ 0.129} \\
   \bottomrule
  \end{tabular}
  \begin{tabular}{lcc}
    \multirow{2}{*}{Dataset} & \multicolumn{2}{c}{Every Other}\\
    & {$n = 5000$} & {$n=2000$}\\
    \midrule
    MADE & 8.373 $\pm$ 0.120 & 8.364 $\pm$ 0.124 \\
    \midrule
    Greedy & 8.334 $\pm$ 0.125 & 8.315 $\pm$ 0.123 \\
    IP & 8.333 $\pm$ 0.129 & \textbf{8.314 $\pm$ 0.123}\\
    IP-var & \textbf{8.331 $\pm$ 0.126} & \textbf{8.314 $\pm$ 0.124}\\
    Zuko & 8.334 $\pm$ 0.126 & 8.315 $\pm$ 0.123\\
    \bottomrule
  \end{tabular}
  \label{tab:mask_algos_density_results}
  }
\end{table}

\subsubsection{Approximate Algorithm Comparisons}
\label{appendix:zuko_comp}

 In this section, we give a more direct example of comparison on autoregressive flow performance between the greedy and Zuko algorithms. While the two factorization schemes of the adjacency matrix are similar in that they approximately maximize the number of remaining paths in the neural network, they can yield different results in specific settings. In the experiment below, mask matrices found by our  Algorithm~\ref{alg:mf} and by Zuko (Algorithm~\ref{alg:zuko}) result in different performance. Zuko operates on the unique rows of an adjacency matrix, which causes issues when the matrix contains many repeating rows. This type of structure may be naturally encountered in datasets with star shaped graphs. Consider a $d$-by-$d$ adjacency matrix where the first ($d-1$) rows contain dependence on the first variable (i.e: $[1, 0, …, 0]$) and only the last row depends on the second variable (i.e.: $[0, 1, 0, …, 0]$). Given a budget of $h$ hidden units, Zuko would assign $\frac{h}{2}$ units to represent the variable corresponding to the last row. This ineffectively represents the dependence of all other outputs on the first variable, and is avoided by our greedy algorithm. In a non-linear dataset with $d$ variables and the above adjacency matrix, we compare StrAF with our greedy algorithm and the Zuko algorithm in the table below. We report runs from 5 random seeds, using a $95\%$ CI and $d$ hidden units. We see our method performs better, especially as $d$ increases. This further supports our claim that the choice of mask factorization can impact performance beyond enforcing a given independence structure.

\begin{table}[h]
\centering
{
\caption{Comparison of StrAF performance using our greedy algorithm and the Zuko algorithm on a star-shaped Bayesian network. Runs from 5 random seeds are reported, using a $95\%$ CI and $d$ hidden units. Evaluation is based on negative log-likelihood on a held-out test set (lower is better)}
\begin{tabular}{lcc}
\multirow{2}{*}{Dataset} & \multicolumn{2}{c}{Star-Shaped}\\
& {$d = 50$} & {$d=1024$}\\
\midrule
Greedy & \textbf{0.4311$\pm$ 0.5405} &\textbf{ -18.734 $\pm$ 1.999} \\
Zuko & 1.3276 $\pm$ 0.2806 & -14.550 $\pm$ 0.936 \\
\bottomrule
\end{tabular}
\label{tab:zuko_v_straf}
}
\end{table}

\section{Additional Experiment Results}

To validate the sample generation quality of StrNN when trained on the MNIST dataset, we display select samples generated by both StrNN and the MADE baseline in Figure~\ref{fig:binary_mnist_samples}. We show that even in the low data regime (models fitted on 1000 training samples), both models generate samples with reasonable quality. Hence, we observe that for the density estimation task, injecting prior structure using a StrNN improves likelihood values for each sample under the model without sacrificing generative quality.

\label{appendix: additional_results}
\begin{figure}[h]
    \begin{subfigure}[b]{\textwidth}
         \centering
         \includegraphics[width=\textwidth]{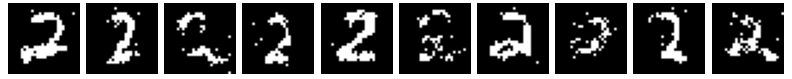}
         \caption{\small Samples generated from trained MADE model.}
     \end{subfigure}
    
     \begin{subfigure}[b]{\textwidth}
         \centering
         \includegraphics[width=\textwidth]{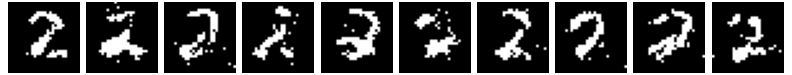}
         \caption{\small Samples generated from trained StrNN model with \texttt{nbr\_size=10}.}
     \end{subfigure}
    \centering
    \caption{\small Sample MNIST handwritten digits (label 2) generated by MADE and StrNN trained on 1000 data points. Both models generate samples of reasonable quality, while StrNN achieves higher likelihoods at the test samples, as illustrated in Figure~\ref{fig:binary}.}
    \label{fig:binary_mnist_samples}
\end{figure}

\section{Synthetic Data Generation}
\label{appendix: synthetic_data_gen}

In this paper, we used the following synthetic datasets:
\begin{enumerate}
    \item $d=800$ binary dataset where each variable depend on every other preceding variable ("Binary every\_other").
    \item $d=50$ binary dataset where the adjacency matrix is randomly generated based on sparsity threshold ("Binary random\_sparse").
    \item $d=20$ Gaussian dataset where each variable is dependent on 2 previous variables ("Gaussian prev\_2").
    \item $d=20$ Gaussian dataset where the adjacency matrix is randomly generated based on a sparsity threshold ("Gaussian random\_sparse"). 
    \item $d=15$ randomly generated non-linear and multi-modal dataset with sparse conditional dependencies between variables used for autoregressive flow evaluation.
    \item $d=5$ and $d=10$ randomly generated datasets following linear SEMs with sparse adjacency matrices for interventional and counterfactual evaluations. The details are described in Appendix~\ref{appendix: causal_data}.
\end{enumerate}

\subsection{Adjacency Structures}
\label{appendix: adjacency}

In this section we visualize the adjacency matrices that were used to generate the synthetic datasets listed above. The procedure to generate each synthetic dataset given these ground truth adjacency matrices are described in the following sections. We also use these adjacency matrices to perform weight masking using StrNN/StrAF, and as the input masks for the Graphical Normalizing Flow. We visualize the adjacency matrices used by the binary experiments in Section \ref{sec:experiment_binary_data} in Figure \ref{fig: binary_adj}.

\begin{figure}[h]
    \centering
    \includegraphics[width=0.7\textwidth]{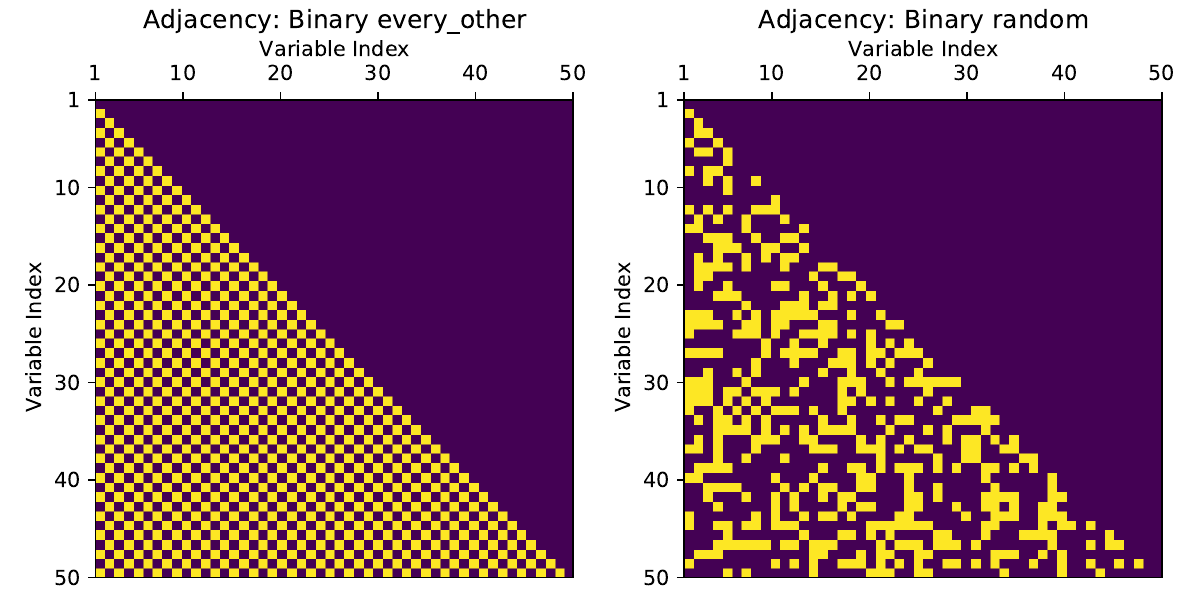}
    \caption{Adjacency matrix used to generate datasets for binary experiments. \textit{Left:} every\_other generation scheme. Note in our actual experiments, we used a 800-dimensional version of this adjacency structure. The 50-dimensional matrix is shown here for illustration only. \textit{Right:} random\_sparse generation scheme. Conditionally dependent variables are shown in yellow. These adjacency matrices are also used to generate StrNN mask matrices.}
    \label{fig: binary_adj}
\end{figure}

The true underlying conditional independence structure for the MNIST dataset used in Section \ref{sec:experiment_binary_data} is unknown, which is also a common challenge for any real-world/image dataset. Instead, when using the StrNN masked neural network, we aim to encode the inductive bias of \textit{locality}, so that density estimation for a single pixel only depends on its surrounding neighbourhood of pixels. For the results shown in the main text, we decided to use a neighbourhood size of 10 after an extensive hyperparameter search for this parameter. As explained briefly in Section~\ref{sec:experiment_binary_data}, the hyperparameter \texttt{nbr\_size} specifies the radius of the square context window originating from each pixel. Each pixel is modelled to be dependent on all previous pixels in that window, as specified by the variable ordering. For variable ordering, we use the default row-major pixel ordering for the MNIST images. The resulting adjacency matrix is visualized in Figure \ref{fig: mnist_adj}.
\begin{figure}[h]
    \centering
    \includegraphics[width=0.7\textwidth]{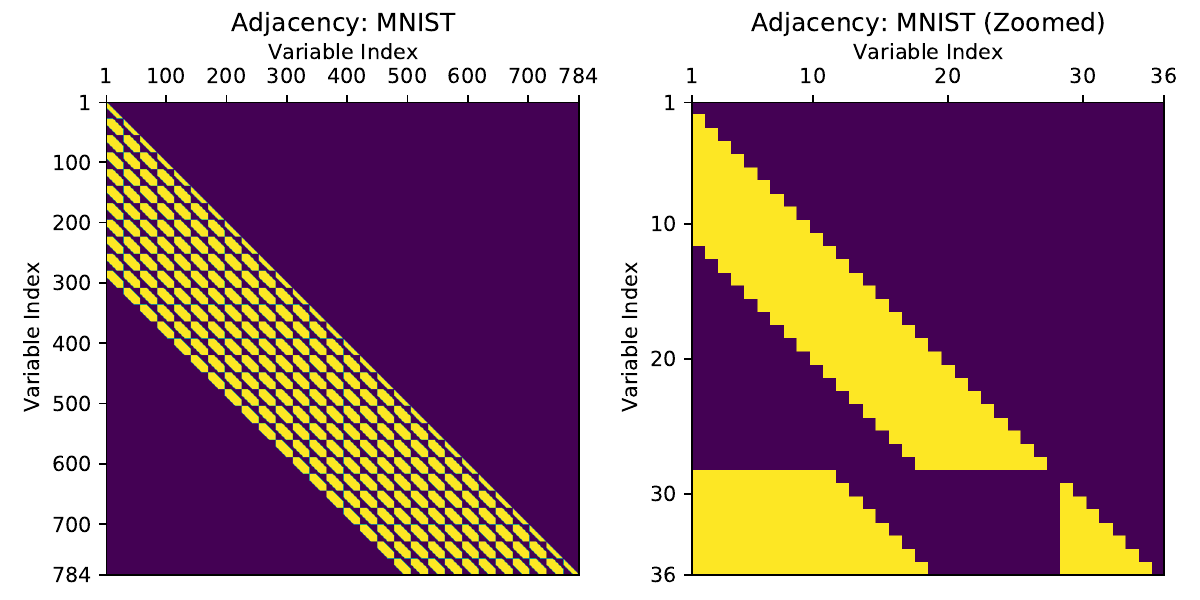}
    \caption{Adjacency matrix used to mask StrNN for the binary MNIST density estimation task, with neighbourhood size set to 10 (\texttt{nbr\_size=10}). Conditionally dependent variables are shown in yellow. \textit{Left:} All variables. \textit{Right:} Zoomed in view of first 36 variables for illustrative purposes.}
    \label{fig: mnist_adj}
\end{figure}

Next, we visualize the adjacency matrices that were used to generate the Gaussian synthetic datasets in Section \ref{sec:experiment_gaussian_data} in Figure \ref{fig: gaussian_adj}. These matrices are also used to generate the mask matrices for the StrNN.

Finally, in Figure \ref{fig: flow_adj} we visualize the adjacency matrix that was used to generate the non-linear multi-modal dataset in Section \ref{sec:straf_experiments}. This adjacency matrix is used to generate masks for StrAF, and is provided to GNF as the ground truth adjacency matrix for its input masking scheme.

\begin{figure}[h]
    \centering
    \includegraphics[width=0.7\textwidth]{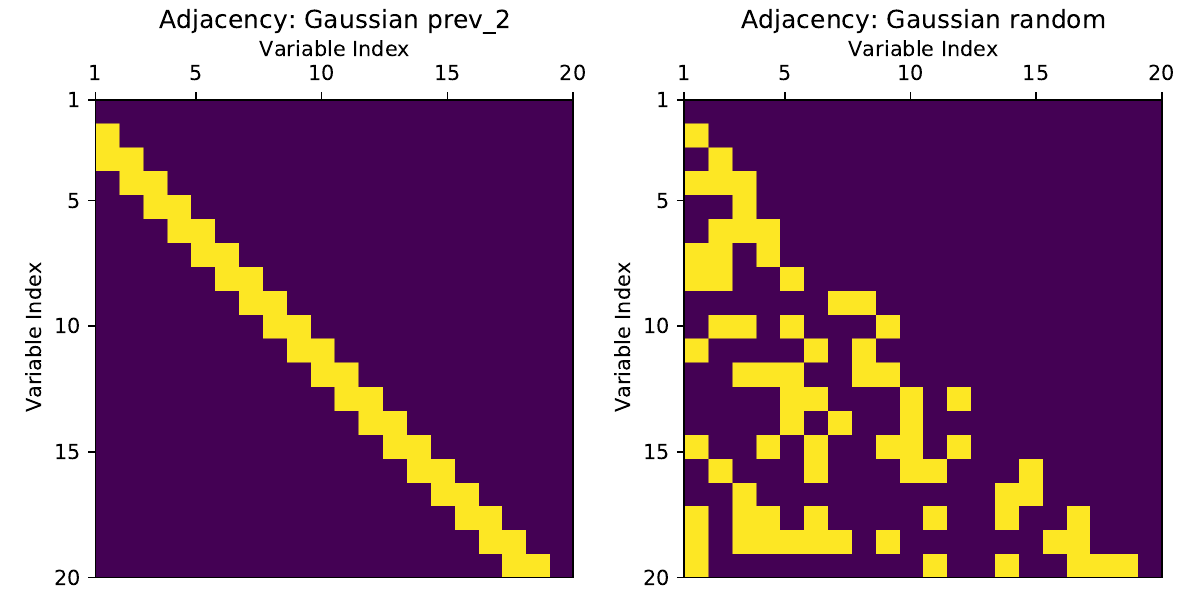}
    \caption{Adjacency matrices used to generate the Gaussian synthetic dataset. Matrices are also used to generate StrNN masks during density estimation tasks. Conditionally dependent variables are shown in yellow. \textit{Left:} prev\_2 generation scheme \textit{Right:} random generation scheme.}
    \label{fig: gaussian_adj}
\end{figure}

\begin{figure}[h]
    \centering
    \includegraphics[width=0.4\textwidth]{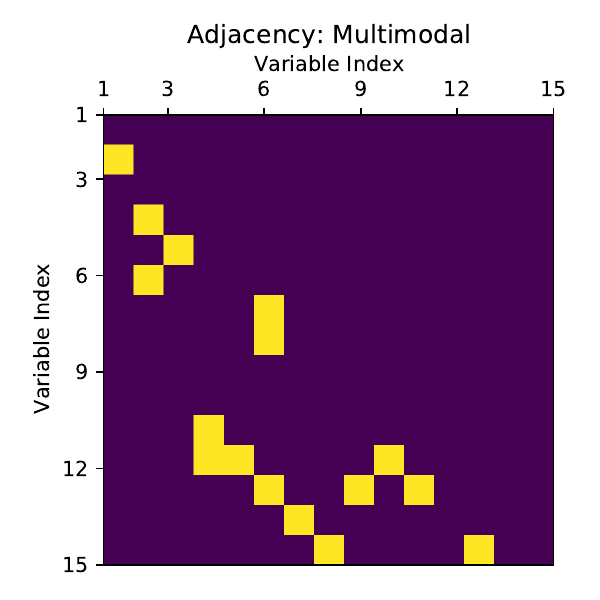}
    \caption{Adjacency matrices used to generate the multi-modal non-linear synthetic dataset used in Section \ref{sec:straf_experiments}. The matrix is also used to generate StrAF masks and for GNF masking. Conditionally dependent variables are shown in yellow.}
    \label{fig: flow_adj}
\end{figure}

\subsection{Binary Synthetic Data}
\label{appendix: binary_data_gen}

We observe that based on the autoregressive assumptions, the synthetic data generating process should draw each $x_i$ as a Bernoulli random variable (i.e., a coin flip) based on $x_1, ..., x_{i-1}$, for $i=1,...,d$. Given an adjacency matrix $A\in \{0,1\}^{d\times d}$, the general structure equations are given by:
\begin{align} \label{eq:binary_generation_equations}
    &x_i \sim \text{Bernoulli}(p_i), \ p_i=\text{Sigmoid}(\sum_{j=1}^{i-1}\alpha_{ij}x_j+c_i),
\end{align}
where $\alpha_{ij} = 0$ if $A_{ij}=0$, otherwise $\alpha_{ij} \sim \mathcal{N}(0, 1)$ and $c_i \sim \mathcal{N}(0, 1)$.
    
\subsection{Gaussian Synthetic Data}
\label{appendix: gaussian_data_gen}
Analogous to binary synthetic data generation, for Gaussian data, we sample each variable as:
\begin{align}
    \label{eq:gaussian_generation_equations}
    &x_i \sim \mathcal{N}(\mu_i, \sigma_i) \text{ where } \mu_i = \sum_{j=1}^{i-1}\alpha_{ij}x_j+c_i, \ \sigma_i \sim \mathcal{N}(0, 1),
\end{align}
where $\alpha_{ij} = 0$ if $A_{ij}=0$, otherwise $\alpha_{ij} \sim \mathcal{N}(0, 1)$.

\subsection{Synthetic Data for Autoregressive Flow Evaluation}
\label{appendix: flow_data_gen}
In this section, we describe the data generating process for the experiments in Section~\ref{sec:straf_experiments}. The objective is to generate a $d=15$-dimensional multi-modal and non-linear dataset. We create the ground truth adjacency matrix $A\in \{0, 1\}^{15\times 15}$ by sampling each entry in the matrix independently from the $\text{Uniform}(0, 1)$ distribution. Each element is converted to a binary value using a sparsity threshold of $0.8$. Moreover, upper triangular elements are then zeroed out. This results in a sparse binary adjacency matrix for which values of one indicate conditional dependence, and zeros indicate conditional independence. The exact matrix is visualized in Appendix \ref{appendix: adjacency}.

In our data generating process, variables with conditional dependencies are generated as a weighted sum of its preceding dependent variables. We generate a second matrix $W \in \mathbb{R}^{15\times 15}$ containing these weights, where each element is sampled from the Uniform$(-3, 3)$ distribution. This matrix $W$ is then multiplied element-wise by $A$ to zero out pairs of variables that are conditionally independent. If we denote entries of $W$ as $w_{ij}$, each dependent pair of variables is generated by the following process:
\begin{equation}
    x_t = \sqrt{\sum_{j=1}^{t-1}(w_{tj}x_{j})^2} + \varepsilon,\quad \varepsilon \sim \mathcal{N}(0, 1).
\end{equation}
Variables which are conditionally independent (e.g., $x_1$) are generated using a mixture of three Gaussians. For each variable and each Gaussian mixture component, we sample its mean from the Uniform$(-8, 8)$ distribution, and its standard deviation from the Uniform$(0.01, 2)$ distribution. We draw the mixture weights from the Dirichlet(1, 1, 1) distribution. At sampling time, we use this mixture weight vector to determine the number of samples to draw from each Gaussian mixture component. For our experiments, we draw $5000$ samples using this data generating process, and use a [0.6, 0.2, 0.2] ratio for training / validation / testing splits.

\section{Causal Inference}
\label{appendix: causal_inference}

\subsection{Algorithms}
\label{appendix: causal_algorithms}
The flow-based models can be easily employed to answer causal queries, and similar to CAREFL \cite{khemakhem2021causal}, we provide the corresponding algorithms using flows for causal queries in this section. For the following algorithms, we denote the forward transformer $\tau$ and the flow $\mathbf{T}$ as the sampling step which transforms the noise $\mathbf{z}$ to data $\mathbf{x}$. Similarly, the backward $\tau^{-1}$ and $\mathbf{T}^{-1}$ indicates the process of taking data to noise. This notation aligns with the setup in Section \ref{sec:structured_CAREFL}.

\subsubsection{Interventions}

To generate samples under intervention $do(x_i=\alpha)$, we can simply modify the corresponding structural equation of the intervened variable by breaking the connection of $x_i$ to $\mathbf{x}_{<\pi(i)}$ and setting $x_i$ to the intervened value of $\alpha$. Then, we can propagate newly sampled latent variables through the modified sampling process of the flow as shown in Algorithm \ref{alg:causal_intv1}.

\begin{algorithm}[hbt!]
\caption{Generate interventional samples (sequential)}\label{alg:causal_intv1}  
\KwData{interventional variable $x_i$, intervention value $\alpha$, number of samples $S$}

for $s=1$ to $S$ do:\\
\qquad sample $\mathbf{z}(s)$ from flow base distribution (the value of $z_i$ can be discarded)\\
\qquad set $x_{i}(s) = \alpha$\\
\qquad for $j = 1$ to total dimension $d$ except $j=i$ do:\\
\qquad \qquad compute $x_{j}(x) = \tau_{j}(z_{j}(s), \mathbf{x}_{<\pi(j)}(s))$ \\
\qquad end for\\
end for\\
return interventional sample $\mathbf{X} = \{\mathbf{x}(s): s=1, ..., S\}$

\end{algorithm}

If the flow only requires one pass for sampling such as inverse autoregressive flows which prevents us from changing the sampling process of each individual dimension, we can generate one random sample $\mathbf{x}$ first and then invert the flow to modify $z_i$ so that the corresponding $x_i$ equals to the intervened value $\alpha$. We additionally give Algorithm \ref{alg:causal_intv2} for flows that only require one pass for sampling, such as inverse autoregressive flows.

\begin{algorithm}[hbt!]
\caption{Generate interventional samples (parallel sampling)}\label{alg:causal_intv2} 
\KwData{interventional variable $x_i$, intervention value $\alpha$, number of samples $S$}

for $s=1$ to $S$ do:\\
\qquad sample $\mathbf{z}(s)$ from flow base distribution (the value of $z_i$ can be discarded)\\
\qquad compute initial sample through $\mathbf{x}(s) = \mathbf{T}(\mathbf{z}(s))$\\
\qquad set $z_{i}(s) = \tau^{-1}_{i}(\alpha, \mathbf{x}(s))$\\
\qquad compute final sample $\mathbf{x}(s) = \mathbf{T}(\mathbf{z}(s))$\\
end for\\
return interventional sample $\mathbf{X} = \{\mathbf{x}(s): s=1, ..., S\}$

\end{algorithm}

\subsubsection{Counterfactuals}

Furthermore, computing counterfactuals is typically more challenging for many causal models as it requires inferring the latent variables $\mathbf{z}$ conditioned on the observed data $\mathbf{x}_{\text{obs}}$. However, with the invertible nature of flows, it becomes straightforward to perform counterfactual inference as we have access to both forward and backward transformations between $\mathbf{x}$ and $\mathbf{z}$. Here, we present Algorithm \ref{alg:causal_ctf} for computing counterfactuals with flows. 

\begin{algorithm}[hbt!]
\caption{Compute counterfactual values} \label{alg:causal_ctf}  
\KwData{observed data $\mathbf{x}_{obs}$, counterfactual variable $x_i$ and value $\alpha$}

infer $\mathbf{z}_{obs}$ from observed data $\mathbf{z}_{obs} = \mathbf{T}^{-1}(\mathbf{x}_{obs})$ (the value of $z_{i}^{obs}$ can be discarded)\\
initialize $\mathbf{z}_{ctf} = \mathbf{z}_{obs}$\\
set $z^{ctf}_{i} = \tau_{i}^{-1}(\alpha, \mathbf{x}^{obs}_{<\pi(i)})$\\
compute counterfactual data $\mathbf{x}_{ctf} = \mathbf{T}(\mathbf{z}_{ctf})$\\

return $\mathbf{x}_{ctf}$

\end{algorithm}

\subsection{Data Generation}
\label{appendix: causal_data}

Here, we describe the data generating process for causal inference evaluations in Section \ref{sec:experiment_causal}. To highlight the benefits from incorporating graphical structures, we generate both 5- and 10-variable SEMs with sparsely dependent variables, following the same procedures as outlined below. 

We first generate the adjacency matrix specifying the dependencies among the variables. Each entry of the matrix is sampled from a Uniform(-2, 2) distribution. We explicitly set any entries with absolute values smaller than $1.5$ along with all upper triangular entries to zero. This creates a sparse matrix with zeros representing the independencies and non-zero elements showing the strength of variable dependencies. The corresponding matrices are visualized in Figure \ref{fig:causal_sem}. 

\begin{figure}[h]
    \begin{subfigure}[b]{0.43\textwidth}
         \centering
         \includegraphics[width=\textwidth]{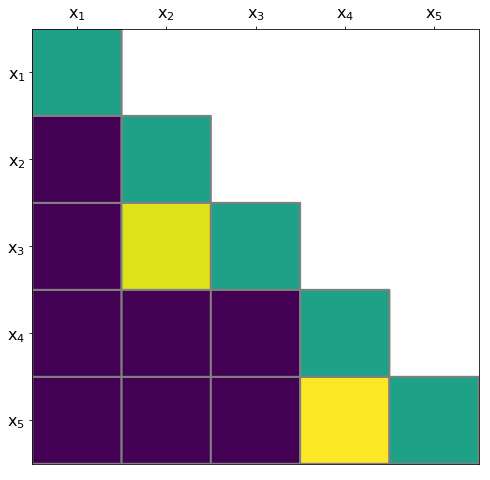}
         \caption{SEM with 5 variables}
     \end{subfigure}
     \begin{subfigure}[b]{0.53\textwidth}
         \centering
         \includegraphics[width=\textwidth]{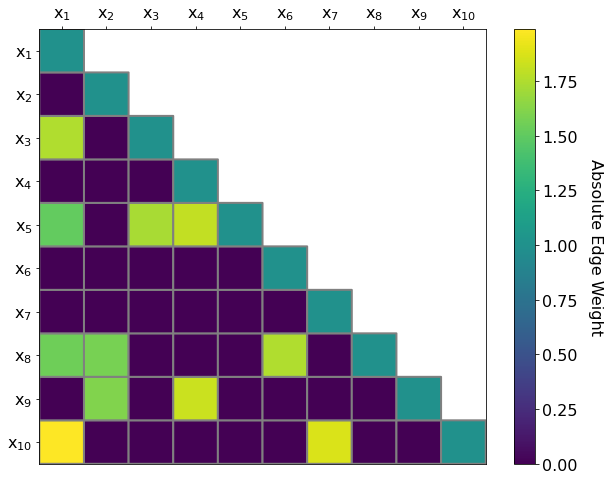}
         \caption{SEM with 10 variables}
     \end{subfigure}
    \centering
    \caption{Adjacency matrices containing the coefficients relating variables in the linear additive SEM. Each row indicates the observed variable, and each column indicates the contribution of that variable to the observed variable. The absolute value of each coefficient is shown, and coefficients with zero weight indicate conditional independence.}
    \label{fig:causal_sem}
\end{figure}

Then, we generate synthetic data according to a linear additive structural equation model. Given an $d$-dimensional variable $\mathbf{x}$ whose distribution is defined by the SEM, each dimension of $\mathbf{x}$ follows the linear relationship:
\begin{equation}
    x_i = \sum^{i-1}_{j=1} w_{ij} x_j + \epsilon_i.
\end{equation}
The weights $w_{ij}$ on the preceding variables came from the generated sparse adjacency matrix, and the additive noise variable $\epsilon_j$ is sampled independently from a standard Gaussian distribution. Observed data corresponding to the SEM is generated via ancestral sampling. We first draw a sample from the noise terms and then sequentially compute the observed variables according to the relationship defined by the adjacency matrix. 

\subsection{Evaluation Metrics}
\label{appendix: causal_metrics}

As mentioned in Section \ref{sec:experiment_causal}, we propose two evaluation metrics, total intervention mean squared error (total I-MSE) and total counterfactual mean squared error (total C-MSE), to comprehensively assess the ability of a causal generative model for answering causal queries. In this section, we present additional details, highlight key distinctions, and provide formal definitions for these metrics.

\subsubsection{Interventions}
We draw interventional samples from both the underlying true SEM where $x_i \sim \mathbb{P}(x_i|do(x_j=\alpha))$ and the causal generative flows where $\bar{x}_i \sim \mathbf{T}(x_i|do(x_j=\alpha))$. To estimate the expectation of the interventional values for the $i$-th dimension, we draw $K=1000$ samples each for $x_i$ and $\bar{x}_i$. We can then compute the following squared error over different sample dimensions $i$, intervened variables $j$, and the corresponding value $\alpha$:
\begin{equation}
    \text{I-Error}(i, j, \alpha) = \left(\frac{1}{K} \sum_{k=1}^{K} x_{i}^{k} - \frac{1}{K} \sum_{k=1}^{K} \bar{x}_{i}^{k}\right)^2.
\end{equation}

We select a set $A$ of intervened values to evaluate the models. The intervened value $\alpha$ came from one of $|A| = 8$ integers perturbed around the mean of the intervened variable $x_i$. Note that we exclude interventions on all preceding variables of the intervened variable as they remain unaffected by the interventions. Hence, the total I-MSE can be obtained by averaging over all possible interventional queries:
\begin{equation}
    \text{Total I-MSE} = \frac{1}{|A| \times d \times (d+1) / 2} \sum^d_{i=1} \sum_{j=1}^i \sum_{\alpha \in A}\text{I-Error}(i, j, \alpha),
\end{equation}
where $d$ is the total dimension of the data $\mathbf{x}$ generated from the SEM.

\subsubsection{Counterfactuals}

Similarly, we compute the counterfactual values $x_i$ from the ground truth as well as $\bar{x}_i$ from the flow models, and we use $i$, $j$, and $\alpha$ to denote the dimension to compute, the counterfactual variable, and the associated value. For counterfactuals, we compute the results based on observed data rather than generating new samples; hence, we collect $K=1000$ observed data $\mathbf{x}_{obs}$ and pose various counterfactual queries to each observed data point. We can derive the following error averaged over different $\mathbf{x}_{obs}$:
\begin{equation}
    \text{C-Error}(i, j, \alpha) = \frac{1}{K} \sum_{k=1}^K \left(x_{i}^{k} -  \bar{x}_{i}^{k}\right)^2.
\end{equation}

We also select a set $A$ of counterfactual values with $|A| = 8$ and exclude evaluations on all preceding variables of the counterfactual variable $x_j$. Then, by considering all possible counterfactual queries, the total C-MSE can be similarly written as:
\begin{equation}
    \text{Total C-MSE} = \frac{1}{|A| \times d \times (d+1) / 2} \sum^d_{i=1} \sum_{j=1}^i \sum_{\alpha \in A}\text{C-Error}(i, j, \alpha),
\end{equation}
where $d$ again is the total dimension of the data $\mathbf{x}$ generated from the SEM.

\subsection{Experimental Setup}

The causal inference experiments are relatively robust to different hyperparameter settings. We mostly follow the same setting as in CAREFL. Each model is trained using the Adam optimizer with learning rate of $0.001$ and  $\beta = (0.9, 0.999)$, along with a scheduler decreasing the learning rate by a factor of 0.1 on plateaux. We train for $750$ epochs with a batch size of $32$ data points. The flow contains 5 layers of sub-flows, and the transformers are 4-layer MADE-like networks with 10 hidden units that incorporate the corresponding autoregressive and graphical structures with masking.

\section{Experiment Setup}
\label{appendix: experiment_setup}

\subsection{Experimental Setup for Binary and Gaussian Density Estimation}
\label{sec:exp_setup_binary_gaussian}

In this section, we describe the experimental setup for the binary and Gaussian density estimation tasks reported in Sections~\ref{sec:experiment_binary_data} and \ref{sec:experiment_gaussian_data}.

\textbf{Method Selection.} We compared StrNN using the greedy mask factorization algorithm (Algorithm~\ref{alg:mf}) to the fully autoregressive MADE baseline as proposed in~\cite{germain2015made}. MADE serves as a natural baseline since both methods use the outputs of an autoencoder to parameterize marginal probabilities, while our StrNN method has the added capability of enforcing additional conditional independence properties. 

\textbf{Training.} Each model is trained with the AdamW optimizer with a batch size of 200 for a maximum of 5000 epochs.
 
\textbf{Hyperparameters.} We employed a grid search to find the optimal hyperparameters for StrNN and MADE respectively, where the grid is provided in Table~\ref{tab:strnn_made_hp_grid}. The number of hidden layers is varied during the hyperparameter search, and the number of hidden units in each hidden layer is determined by the input dimension $d$ times the hidden size multiplier of that layer. The Best hyperparameters for each model, dataset, and sample size discussed in Sections~\ref{sec:experiment_binary_data} and \ref{sec:experiment_gaussian_data} are not listed here since there are too many combinations. Please refer to the code repositories for reproducing the results.

\textbf{Evaluation Metrics.} Results from binary experiments are reported in terms of the negative log-likelihood (NLL) in Figures~\ref{fig:binary} and~\ref{fig:gaussian_synth}. Note that in the binary case, the NLL can simply be rewritten as the binary cross-entropy loss:
\begin{equation}
\label{eq:cross_entropy}
-\log p(\mathbf{x}) 
= \sum^d_{j=1} -\log p(x_j|\mathbf{x}_{<j})
= \sum^d_{j=1} -x_j \log \hat{x}_j 
- (1-x_j)\log (1 - \hat{x}_j).
\end{equation} 
The results from the Gaussian experiments are also reported in NLL, which is calculated by using the neural network outputs as the parameters in each marginal conditional of the Gaussian distribution. The error ranges for the results from these experiments are computed as standard error across samples in the held-out test set.

\begin{table}[h]
    \small
    \centering
    \begin{tabular}{l|l}
        Hyperparameter & Grid Values \\
        \hline
        Activation & [relu] \\
        Epsilon & [1, 0.01, 1e-05] \\
        Hidden size multiplier 1 & [1, 4, 8, 12] \\
        Hidden size multiplier 2 & [1, 4, 8, 12] \\
        Hidden size multiplier 3 & [1, 4, 8, 12] \\
        Hidden size multiplier 4 & [1, 4, 8, 12] \\
        Hidden size multiplier 5 & [1, 4, 8, 12] \\
        Number of hidden layers & [1, 2, 3, 4, 5] \\
        Learning rate & [0.1, 0.05, 0.01, 0.005, 0.001] \\
        Weight decay & [0.1, 0.05, 0.01, 0.005, 0.001]
    \end{tabular}
    \caption{\small Hyperparameter grid: StrNN vs. MADE}
    \label{tab:strnn_made_hp_grid}
\end{table}

\subsection{Experimental Setup for Normalizing Flow Evaluations}
\label{appendix: nf_setup}
In this section, we describe the hyperparameters and metrics used for the experimental evaluation between the Structured Autoregressive Flow (StrAF), Structured Continuous Normalizing Flow (StrCNF), and other flows in Section \ref{sec:straf_experiments}.

\textbf{Method Selection.} We compared the StrAF against a fully autoregressive flow (denoted ARF in the main text) and the Graphical Normalizing Flow (GNF) \cite{wehenkel2021graphical}. The fully autoregressive flow assumes no conditional independencies in the data generation process, and uses a full lower triangular adjacency matrix for masking. Meanwhile, the GNF model encodes conditional independencies, and we provide it access to the true adjacency matrix. For continuous flows, we compare the StrCNF against the FFJORD baseline and the model provided by \cite{weilbach20sccnf}. The FFJORD baseline uses a fully connected neural network to represent the flow dynamics, while the \cite{weilbach20sccnf} model can inject structure only into neural networks with a single hidden layer. This can be circumvented by stacking these neural networks, but this is shown to do worse than StrCNF in our experiments. We evaluate each model on the synthetic data described in Appendix \ref{appendix: flow_data_gen}.

\textbf{Training.}
Each model is trained using the Adam optimizer for a maximum of 150 epochs using a batch size of 256. During all runs, the models were trained using early stopping on the validation log-likelihood loss with a patience of 10 epochs, after which the model state at the best epoch was selected. We consider two additional training schedules: decreasing the learning rate by a factor of 0.1 on plateaus where the loss does not improve for five epochs (denoted Plateau), and a single scheduled decrease by a factor of 0.1 at epoch 40 (denoted MultiStep). In addition to the standard fixed learning rate, we select between these training schedules as a hyperparameter.

While the GNF can learn an adjacency matrix from data, we are interested in scenarios where an adjacency matrix is prescribed. Thus, we disable the learning functionality of the adjacency in GNF by stopping gradient updates to the GNF input mask matrix. We retain the one hot encoding network described in the GNF paper. The fully autoregressive flow is implemented by using a GNF with a full lower triangular adjacency matrix. Latent variables are permuted in ARF and GNF, as described by their original publications, but we do not permute variables for StrAF.

\textbf{Discrete Flow Hyperparameters.}
Here we report the process used to select model hyperparameters for discrete flows. We use the UMNN~\citep{NEURIPS2019_2a084e55} as each flow's transformer. We use 20 integration steps to compute the transformer output. The UMNN is conditioned on values computed by the conditioner. We select the dimension of these values as a hyperparameter named "UMNN Hidden Size". We then determine the best hyperparameters for each method using a grid search with the values in Table~\ref{tab:straf_hyperparameter_grid}.
\begin{table}[h]
    \small
    \centering
    \begin{tabular}{l|l}
        Hyperparameter & Grid Values \\
        \hline
        Flow Steps & $[1, 5, 10]$ \\
        Conditioner Net Width & $[25, 50, 500, 1000]$ \\
        Conditioner Net Depth & $[2, 3, 4]$ \\
        UMNN Hidden Size & $[10, 25, 50]$ \\
        UMNN Width & $[100, 250, 500]$ \\
        UMNN Depth & $[2, 4, 6]$ \\
        Learning Rate & $[0.001, 0.0001]$\\
        LR Scheduler & [Fixed, Plateau, MultiStep]
    \end{tabular}
    \caption{\small Hyperparameter grid for discrete flows}
\label{tab:straf_hyperparameter_grid}
\end{table}

The best hyperparameters (as determined by validation loss) that were selected for each model used in the main text are reported  in Table~\ref{tab:flows_hyperparameters}.
\begin{table}[h]
    \small
    \centering
    \begin{tabular}{l|lll}
        Hyperparameter & ARF & GNF & StrAF\\
        \hline
        Flow Steps & 5 & 5 & 10 \\
        Conditioner Net Width & 500 & 500 & 500 \\
        Conditioner Net Depth & 4 & 2 & 2\\
        UMNN Hidden Size & 25 & 50 & 25\\
        UMNN Width & 250 & 250 & 250 \\
        UMNN Depth & 6 & 6 & 6 \\
        Learning Rate & 0.001 & 0.001 & 0.001\\
        LR Scheduler & Plateau & Plateau & Plateau
    \end{tabular}
    \caption{Final hyperparameters used per flow model}
    \label{tab:flows_hyperparameters}
\end{table}

We fix the best hyperparameters for each model, and then re-train each model on the same data splits using eight new random seeds. This ensemble of eight models was used to generate confidence intervals for the test NLL.

\textbf{Continuous Flow Hyperparameters.}
Here we report the hyperparameters used in the continuous normalizing flow models. We use the FFJORD implementation of the CNF as a baseline, and then substitute the neural network representing dyanmics in the architecture with the \cite{weilbach20sccnf} masked neural network, or the StrNN. For the ODE solver, we generally use the default FFJORD hyperparameters, as reported in Table \ref{tab:cont_flow_fixed_hp}.
\begin{table}[h]
    \small
    \centering
    \begin{tabular}{l|l}
        Hyperparameter & Values\\
        \hline
        ODE-NN Activation Function & Tanh \\
        ODE Time Length & 0.5 \\
        ODE Train Time & True \\
        ODE Solver & dopri5 \\
        ODE Solver Absolute Tolerance & 1e-5 \\
        ODE Solver Relative Tolerance & 1e-5  \\
    \end{tabular}
    \caption{Fixed hyperparameters for CNF models.}
    \label{tab:cont_flow_fixed_hp}
\end{table}

Other hyperparameters were selected using validation loss from the following grid in Table \ref{tab:cont_flow_hp}. We again note that the neural network representing ODE dynamics for the Weilbach model is adapted to have multiple hidden layers by stacking multiple neural networks that have a single hidden layer. The final hyperparameters for each model are reported in Table \ref{tab:cont_flow_hp_final}. As before, the final test NLL results in the main text are reported after fixing hyperparameters, and then training 8 models from random initialization.
\begin{table}
\begin{minipage}{.49\linewidth}
    \small
    \centering
    \begin{tabular}{l|l}
        Hyperparameter & Grid\\
        \hline
        Flow Steps & [1, 5, 10] \\
        ODE-NN Width & [50, 500] \\
        ODE-NN Depth & [2, 3, 8] \\
        Learning Rate & [5e-3, 1e-3, 1e-4]
    \end{tabular}
    \caption{Hyperparameter grid for CNF models.}
    \label{tab:cont_flow_hp}
\end{minipage}\hfill
\begin{minipage}{.49\linewidth}
    \small
    \centering
    \begin{tabular}{l|lll}
        Hyperparameter & FFJORD & Weilbach & StrCNF\\
        \hline
        Flow Steps & 5 & 10 & 10\\
        ODE-NN Width & 50 & 500 & 500\\
        ODE-NN Depth & 2 & 3 & 2\\
        Learning Rate & 5e-3 & 5e-3 & 5e-3
    \end{tabular}
    \caption{Final hyperparameters for CNF models.}
    \label{tab:cont_flow_hp_final}
\end{minipage}
\end{table}

Importantly, the adjacency matrix for the StrCNF must be modified to set the main diagonal to be all ones. Intuitively, this allows each variable to interact with its prior value in the dynamics. This was also applied to the Weilbach model. The choice to fill the main diagonal of the main adjacency matrix was made using validation performance. Alternatives such as using the original adjacency or reflecting the adjacency about the main diagonal performed worse on the validation data.

\pagebreak
\subsection{Graphical Normalizing Flow Baseline}
\label{appendix: gnf}
We use the official implementation of Graphical Normalizing Flows (GNF) as provided in the publication~\cite{wehenkel2021graphical}. In our experiments, we found issues when using the code to evaluate the reverse flow, i.e., the transformation from the latent space to the data distribution. The code we believe to be problematic has been copied below with minor edits for clarity. The original version can be found at: \url{https://github.com/AWehenkel/Graphical-Normalizing-Flows/blob/2cc6fba392897ec1884b4f01a695b83d3c04883a/models/NormalizingFlow.py#L166}.
\begin{verbatim}
    def forward(self, x):
        inv_idx = torch.arange(x.shape[1] - 1, -1, -1).long()
        for step in self.steps:
            z = step(x)
            x = z[:, inv_idx]
        return z

    def invert(self, z):
        for step in range(len(self.steps)):
            z = self.steps[-step].invert(z)
        return z
\end{verbatim}

Typically, if we denote the forward flow transformation as $\mathbf{T} = \mathbf{T}_1 \circ ... \circ \mathbf{T}_K$, then the reverse flow should be computed as $\mathbf{T}^{-1} = \mathbf{T}^{-1}_K \circ ... \circ \mathbf{T}^{-1}_1$. The GNF code implemented in Python, however, computes the reverse flow as $\mathbf{T}^{-1} = \mathbf{T}^{-1}_1 \circ \mathbf{T}^{-1}_K \circ \mathbf{T}^{-1}_{K-1} ... \circ \mathbf{T}^{-1}_2$. These reverse flow transformations are only the same when $\mathbf{T}_1 = \textbf{Id}$, but in general there is no guarantee they are equivalent when learning all layers in the flow, i.e., $\mathbf{T}_1 \neq  \textbf{Id}$.

More importantly, while the forward transformation permutes latent variables between flow steps, the inverse transformation does not perform the same permutation. We believe that this choice harms sample quality generation, as can be seen in our experimental section. However, we could not find reference to the inversion process in the GNF publication, and thus were unable to determine if this was a bug in their code, or an actual limitation of the method. As such, we decided to simply use the official GNF code to produce the results for our evaluations.

\end{document}